%% file: kdd2025.tex
\documentclass[sigconf]{acmart} 
\AtBeginDocument{%
  }

\setcopyright{acmlicensed}
\copyrightyear{2018}
\acmYear{2018}
\acmDOI{XXXXXXX.XXXXXXX}
\acmConference[Conference acronym 'XX]{Make sure to enter the correct
  conference title from your rights confirmation email}{June 03--05,
  2018}{Woodstock, NY}
\acmISBN{978-1-4503-XXXX-X/2018/06}

\usepackage{times}
\usepackage{helvet}
\usepackage{courier}
\usepackage{pgfplots}
\usepackage{graphicx}
\usepgfplotslibrary{groupplots}
\pgfplotsset{compat=newest}
\usepackage{comment}
\usepackage{amsmath}
\usepackage{amsthm}
\usepackage{amsfonts}
\usepackage{centernot}
\usepackage{tikz}
\usepackage{balance}

\usepackage{algorithmic}
\usepackage{algorithm} 
\usepackage{graphicx}
\usepackage{natbib}
\usepackage{booktabs}
\usepackage{siunitx}
\usepackage{thmtools} 
\usepackage{thm-restate}
\usepackage{lipsum}
\usepackage{multirow}
\usepackage{mathtools}
\usepackage{graphicx}
\usepackage{booktabs}
\usepackage{adjustbox}
\usepackage{array}
\usepackage{pgfplots}
\pgfplotsset{compat=newest}
\usepackage{float}
\usepackage{soul}
\usepackage{xcolor}

\usepackage{cleveref}
\usepackage{pgfplots}
\usepackage{caption}
\usepackage{subcaption}

\usepackage{pifont}
\usepackage{pgfplots}
\pgfplotsset{compat=1.18}
\usepgfplotslibrary{groupplots}
\definecolor{darkblue}{rgb}{0, 0.2, 0.4}

\newcommand{\sd}[1]{{{\footnotesize±}{\scriptsize#1}}}

\usepackage{url}
\usepackage{hyperref} 
 \hypersetup{ 
     colorlinks=true, 
     linkcolor=black, 
     filecolor=black, 
     citecolor = black,       
     urlcolor=blue, 
     pdfborderstyle={/S/U/W 1}
     }

\begin{document}


\title{Depth-Adaptive Graph Neural Networks via Learnable Bakry-Émery 
 Curvature}

\author{Asela Hevapathige}
\email{asela.hevapathige@anu.edu.au}
\affiliation{%
  \institution{School of Computing \\ Australian National University}
  \city{Canberra}
  \country{Australia}
}

\author{Ahad N. Zehmakan}
\email{ahadn.zehmakan@anu.edu.au}
\affiliation{%
  \institution{School of Computing \\ Australian National University}
  \city{Canberra}
  \country{Australia}
}

\author{Qing Wang}
\email{qing.wang@anu.edu.au}
\affiliation{%
  \institution{School of Computing \\ Australian National University}
  \city{Canberra}
  \country{Australia}
}

\begin{abstract}
Graph Neural Networks (GNNs) have demonstrated strong representation learning capabilities for graph-based tasks. Recent advances on GNNs leverage geometric properties, such as curvature, to enhance its representation capabilities by modeling complex connectivity patterns and information flow within graphs. However, most existing approaches focus solely on discrete graph topology, overlooking  diffusion dynamics and task-specific dependencies essential for effective learning. To address this, we propose integrating Bakry-Émery curvature, which captures both structural and task-driven aspects of information propagation. We develop an efficient, learnable approximation strategy, making curvature computation scalable for large graphs. Furthermore, we introduce an adaptive depth mechanism that dynamically adjusts message-passing layers per vertex based on its curvature, ensuring efficient propagation. Our theoretical analysis establishes a link between curvature and feature distinctiveness, showing that high-curvature vertices require fewer layers, while low-curvature ones benefit from deeper propagation. Extensive experiments on benchmark datasets validate the effectiveness of our approach, showing consistent performance improvements across diverse graph learning tasks. 
\end{abstract}

\maketitle

\input{sections/introduction}
\input{sections/related_work}

\input{sections/concepts}
\input{sections/adaptive_gnn}
\input{sections/optimization}

\input{sections/theoretical_analysis}
\input{sections/experiments}

\input{sections/conclusion}

\clearpage
\balance
\bibliography{references}
\bibliographystyle{ACM-Reference-Format}
\clearpage
\input{sections/appendix}

\end{document}

%% file: sections/introduction.tex
\section{Introduction}
Graph Neural Networks (GNNs) have emerged as a transformative technique for learning graph representations by combining vertex features with structural data. Their strength in capturing relational dependencies has driven their use in various tasks such as vertex classification~\cite{wu2020comprehensive,xiao2022graph}, link prediction~\cite{zhang2018link,zhang2022graph}, and community detection~\cite{chen2019supervised,souravlas2021survey}. Typically, GNNs stack multiple message-passing layers, allowing vertices to iteratively aggregate information from their neighbors and improve their representations. This design makes GNNs effective for modeling complex graph structures.


Recently, researchers have leveraged geometric tools, such as curvature~\cite{bochner1946vector}, to enhance the representational power of GNNs~\cite{toppingunderstanding, ye2019curvature}. Originally rooted in differential geometry, curvature measures how a space deviates from flatness by tracking the convergence or divergence of geodesics~\cite{bochner1946vector,ollivier2007ricci}. When applied to graphs, it captures the connectivity and relational properties of vertices and edges~\cite{lin2011ricci, ni2015ricci}, offering a more nuanced structural analysis. Incorporating curvature into GNNs can not only deepen our understanding of local connectivity but also provide insights into information flow, helping mitigate issues like oversoothing and oversquashing \cite{sun2023deepricci,fesser2024mitigating}. Among the various forms of curvature, Ricci curvature has gained prominence due to its formulation on discrete spaces, which allows for efficient computation. This efficiency makes it well-suited for analyzing the structural properties of graphs~\cite{forman1999combinatorial,ollivier2007ricci,samal2018comparative,fesser2024mitigating}.


Despite their promise, current approaches compute curvature as a pre-processing step for GNNs, limiting its integration to decoupled techniques, such as graph rewiring and sampling~\cite{toppingunderstanding,fesser2024mitigating,liu2023curvdrop,nguyen2023revisiting,ye2019curvature,li2022curvature}, which do not interact dynamically with the learning process. Consequently, these methods often overlook the interplay between vertex features, structured data, and the evolving dynamics of information diffusion that directly impact downstream tasks. However, integrating curvature dynamically during training presents significant challenges. Traditional curvature measures rely on the static, discrete graph structure, limiting their flexibility in learning contexts. Moreover, dynamic integration requires efficient curvature computation at every training step and the ability to adapt to the continuously changing training process.


 In this work, we address these challenges by leveraging Bakry-Émery curvature \cite{cushing2020bakry, mondal2024bakry, pouryahya2016bakry}. Unlike traditional curvature measures, Bakry-Émery curvature captures intrinsic graph structures while accounting for diffusion dynamics, vertex function behavior, and gradient flows. Essentially, it extends Ricci curvature by incorporating a broader range of graph features, offering a more nuanced view of structural properties and information propagation~\cite{wei2009comparison}. Notably, although Bakry-Émery curvature is computed locally, its lower bounds provide theoretical guarantees for global properties such as Markov chain convergence rates and functional inequalities \cite{weber2021entropy}. These features make it particularly well-suited for GNN architectures, where understanding the interplay between structure and diffusion can lead to more effective learning strategies.

However, integrating Bakry-Émery curvature into the GNN architecture is non-trivial. One key difficulty arises from its definition~\cite{ambrosio2015bakry}, which requires evaluating curvature over the entire function space, a computationally intractable task for large, complex graphs. To overcome this, we propose a learnable strategy that selects a task-specific subset of functions, enabling efficient curvature estimation by focusing on the most relevant structural and diffusion characteristics. Moreover, we introduce an adaptive mechanism that leverages the estimated vertex curvature to dynamically adjust the depth of message-passing layers for each vertex. This approach not only enhances the feature distinctiveness of GNN architectures but also captures finer structural variations within the graph.

To summarize, our main contributions are as follows:

\begin{itemize}
    \item \textbf{A New Perspective on Curvature in GNNs:} We propose Bakry-Émery curvature as a geometric tool for measuring vertex curvature in GNNs, capturing both structural properties and diffusion dynamics of graphs.

    \item \textbf{Efficient Curvature Approximation:}  
We develop a learnable strategy that selects a task-specific subset of functions, enabling efficient estimation of Bakry-Émery curvature on large-scale graphs.

    \item \textbf{Depth-Adaptive Mechanism for GNNs:}  
    We introduce an adaptive mechanism that uses estimated Bakry-Émery curvature to dynamically determine the message-passing depth for each vertex, thereby enhancing the feature distinctiveness of existing GNN architectures.

        \item \textbf{Theoretical Insights:} 
        We theoretically link Bakry-Émery curvature to feature distinctiveness in GNNs, showing that high-curvature vertices (fast diffusion, short mixing times) need fewer layers for distinct representations, while low-curvature vertices (slow diffusion, long mixing times) benefit from deeper architectures.

    \item \textbf{Enhanced GNN Performance:}  
   Extensive experiments show that integrating our depth-adaptive mechanism consistently improves the performance of existing GNN architectures across a range of downstream tasks.\looseness=-1
\end{itemize}


%% file: sections/related_work.tex
\section{Related Work}

\subsection{Curvature-based GNNs}

Curvature, originally defined on smooth manifolds via geodesics, has been successfully extended to discrete structures like graphs~\cite{lee2013introduction,ollivier2009ricci}. Recent work in GNNs has leveraged various notions of curvature, most notably discrete Ricci curvature, to capture structural relationships in graphs. For instance, \citet{ye2019curvature} employed discrete Ricci curvature to quantify the relationship between the neighborhoods of vertex pairs, using these values to define edge weights and enhance message passing. Subsequent studies \cite{toppingunderstanding, nguyen2023revisiting, liu2023curvdrop, fesser2024mitigating} have further explored how curvature influences common GNN challenges: oversmoothing (where vertex representations become overly similar, typically linked to regions of positive curvature) and oversquashing (where long-range information is inadequately propagated, often associated with negative curvature). Building on these insights, \citet{toppingunderstanding, nguyen2023revisiting, fesser2024mitigating} introduced graph rewiring strategies to improve information flow, while \citet{liu2023curvdrop} proposed an edge sampling mechanism that leverages curvature to alleviate these issues. However, a common limitation of these approaches is their treatment of curvature as a fixed, intrinsic property derived solely from the graph topology (i.e., precomputed before training), which can restrict adaptability and generalization in learning scenarios. More recently, \citet{chen2025graph} proposed a curvature-based diffusion mechanism that iteratively updates edge curvature, effectively treating curvature as a learnable parameter. Although this approach introduces greater flexibility, its lack of a rigorous theoretical foundation for curvature learning raises concerns about potential instability and overfitting, as the model must concurrently infer both the curvature and the underlying graph structure.

Our work overcomes these limitations by adaptively learning curvature tailored to the downstream task, enhancing generalizability. We provide a theoretically grounded framework that captures both graph geometry and the behavior of information propagation functions over the graph, accounting for the graph's topology and diffusion dynamics, thereby offering a principled measure of information propagation to guide the learning process. Unlike previous methods that define curvature at the edge level, we define curvature at the vertex level, offering a more nuanced representation of the graph structure.

\subsection{GNNs with Adaptive Architectures}

GNNs with adaptive architectures learn key components of their architecture during training, allowing them to optimize both structure and operations based on the graph's properties and the downstream task. Unlike conventional GNNs with fixed designs, these models incorporate learnable components that dynamically adjust aggregation strategies, assign importance to vertices/edges, and refine connectivity patterns. Numerous studies have explored adaptive behaviors of GNNs in the literature. For example, sampling-based methods selectively extract graph structures to improve computational efficiency and representation capabilities  \cite{younesian2023grapes,zhao2023learnable,lee2023towards}. 
Similarly, several studies focus on defining adaptive neighborhoods for aggregation at each vertex by utilizing learnable mechanisms that are governed by the ground truth values of the downstream task \cite{saha2023learning, guang2024graph}. Attention-based GNNs \cite{velivckovic2018graph,brodyattentive} adjust edge weights adaptively by learning attention coefficients that reflect the relative importance of neighboring vertices during the learning process. Additionally, several studies have proposed dynamic message-passing paradigms that can adaptively determine vertex message-passing operations, such as message filtering and the direction of information propagation \cite{errica2023adaptive,finkelshteincooperative}.

Our approach differs fundamentally from existing adaptive GNN methods by introducing an adaptive layer depth mechanism guided by learnable vertex curvatures. Specifically, we define per-vertex adaptive depth, allowing the model to adjust the number of message-passing steps for each vertex according to its curvature.


%% file: sections/concepts.tex
\section{Background}



Let \(G = (V, E, w)\) be an undirected weighted graph, where \(V\) is the set of vertices, \(E\) is the set of edges, and \(w: E \to \mathbb{R}_{+}\) is the edge weight function. For any vertex \(x \in V\), we denote by \(N(x)\) the set of vertices adjacent to \(x\).

\subsection{Graph Neural Networks}


Graph Neural Networks (GNNs) \cite{wu2019simplifying} extend deep learning to non-Euclidean domains by learning directly from graph-structured data. The typical GNN architecture follows a message-passing framework, consisting of two main operations:

\[
\begin{aligned}[t]
\text{Aggregation: } \quad m_x^{(t)} &= \text{AGG}\left(\{\!\!\{ h_y^{(t-1)} \mid y \in N(x) \}\!\!\}\right),\\[1mm]
\text{Update: } \quad h_x^{(t)} &= \text{UPD}\left(h_x^{(t-1)}, m_x^{(t)}\right).
\end{aligned}
\]
For each vertex \(x\) at layer \(t\), the aggregated message \(m_x^{(t)}\) is computed from the features of its neighbors \(N(x)\) using the aggregation function \(\text{AGG}\). The vertex representation \(h_x^{(t)}\) is then updated by combining the previous state \(h_x^{(t-1)}\) with the aggregated message via the update function \(\text{UPD}\). Different choices for these functions yield various GNN architectures \cite{xu2018powerful, velivckovic2018graph, kipf2016semi, hamilton2017inductive}.

In this work, we leverage Bakry-Émery curvature to characterize underlying diffusion patterns and structural properties of a graph. We incorporate curvature into existing GNN architectures in a learnable manner through a depth-adaptive layer mechanism to enhance their representational capacity.

\subsection{Bakry-Émery Curvature Formulation}
Bakry–Émery curvature provides a way to capture the local geometric behavior of functions on a graph~\cite{ambrosio2015bakry}. We begin by defining a few operators that are fundamental to this concept.



For a function \(f: V \to \mathbb{R}\), its weighted Laplacian at \(x\) is defined as
\begin{equation}
\Delta f(x) = \sum_{y \in N(x)} w(x, y) (f(y) - f(x)).
\label{eq1}
\end{equation}
This operator captures the weighted difference between \(f(x)\) and the values of \(f\) at its neighbors. 


To formalize the Bakry–Émery curvature, we introduce two operators that capture the local behavior of functions on the graph. The first operator, \( \Gamma(f, f)(x) \), analogous to the squared gradient, quantifies the local variability of \(f\) at \(x\) relative to its neighbors:
\begin{equation}
 \Gamma(f, f)(x) = \frac{1}{2} \left(\sum_{y \in N(x)} w(x, y) (f(y) - f(x))^2\right)
 \label{eq2}
\end{equation}
The second operator, \( \Gamma_2(f,f)(x) \), quantifies the convexity of \( f \) at \( x \) as follows:
\begin{equation}
 \Gamma_2(f, f)(x) = \frac{1}{2} \Delta \Gamma(f, f)(x) - \Gamma(f,\Delta f)(x)
 \label{eq3}
\end{equation}
which can be equivalently expanded as,
\begin{equation*}
\begin{aligned}
\Gamma_2(f, f)(x) = & \frac{1}{2} \bigg(\sum_{y \in N(x)} w(x, y) \big(\Gamma(f, f)(y) - \Gamma(f, f)(x)\big) \bigg) \\
&- \sum_{y \in N(x)} w(x, y) (f(y) - f(x)) \big(\Delta f(y) - \Delta f(x)\big).
\end{aligned}
\end{equation*}
The Bakry–Émery curvature \(\kappa(x)\) at vertex \(x\) is defined as the largest constant such that the curvature–dimension inequality
\[
\Gamma_2(f,f)(x) \ge \kappa(x)\,\Gamma(f,f)(x)
\]
holds for all functions \(f: V \to \mathbb{R}\). Equivalently, \(\kappa(x)\) represents the tightest lower bound on the ratio
\[
\frac{\Gamma_2(f,f)(x)}{\Gamma(f,f)(x)},
\]
for all functions \(f\) with \(\Gamma(f,f)(x) \neq 0\). This lower bound reflects the local geometric and functional properties of the graph at the vertex \(x\).\looseness=-1







%% file: sections/adaptive_gnn.tex
\begin{figure*}[]
  \centering  \includegraphics[width=1\textwidth, height=0.23\textheight]{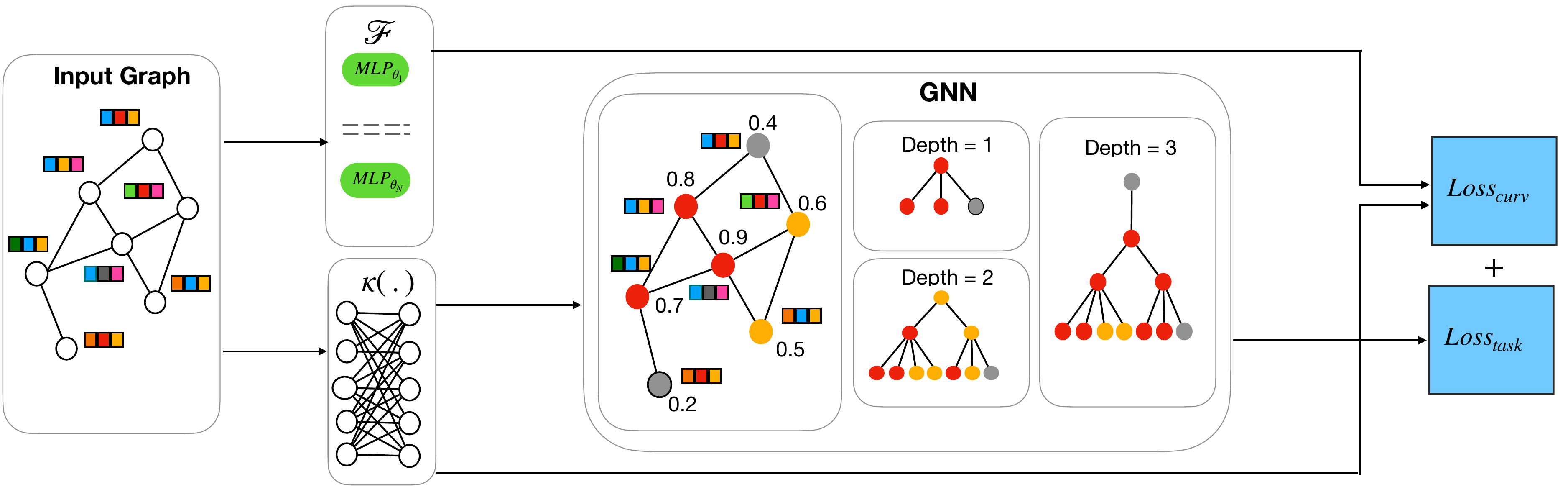}
  \caption{High-level overview of the proposed model: Vertices are clustered based on learned curvature, where message-passing depth is adaptively adjusted according to curvature. Both the curvature functions and the GNN are jointly optimized via a loss function.}
  \label{fig:model summary}
\end{figure*}

\section{Curvature and Information Propagation}
In this section, we examine the influence of Bakry–Émery curvature on local information diffusion in a graph. In particular, we establish a relationship between curvature and the local mixing time, which quantifies the rate at which information equilibrates around a vertex.
The proof is included in the appendix.

Given a vertex \(x \in V\) with Bakry–Émery curvature \(\kappa(x)\), we define the local mixing time \(\tau_x(\epsilon)\) for any tolerance \(\epsilon > 0\) as 
\[
\tau_x(\epsilon) = \inf\left\{ t > 0 : |\nabla f_t|^2(x) \leq \epsilon\,|\nabla f_0|^2(x) \text{ for all } f_0 \in \ell^2(V) \right\}
\]
Here, \( \ell^2(V) \) denotes the space of square-summable functions on \(V\), and \( f_t = e^{-tL} f_0 \) describes the evolution of \( f_0 \) under the heat semigroup associated with the Laplacian \(L\). The operator \(e^{-tL}\) is defined via the matrix exponential. 

The local gradient at \(x\) is defined by
\[
|\nabla f|(x) \coloneqq \sqrt{\Gamma(f,f)(x)}.
\]

\begin{restatable}[Mixing Time Bound]{theorem}{thmmt}\label{thm:mixing}
Let \(G = (V, E, w)\) be an undirected, weighted graph with bounded degree and Laplacian \(L\). If the local curvature-dimension inequality holds for every \(f \in \ell^2(V)\) at vertex \(x\), then for any \(\epsilon \in (0,1)\), the local mixing time satisfies
\[
\tau_x(\epsilon) \le \frac{\log(1/\epsilon)}{\kappa(x)}.
\]
\end{restatable}

This theorem indicates that higher Bakry–Émery curvature at a vertex accelerates the local decay of gradients, thereby enabling faster information diffusion within its neighborhood. In essence, vertices with high curvature act as efficient hubs, quickly equilibrating information and promoting rapid propagation throughout the graph. Conversely, low-curvature vertices, which are less effective at mixing, lead to slower information spread. Since the mixing time is inversely proportional to curvature, enhancing curvature in the network can be seen as a mechanism for improving diffusion efficiency.




\section{Depth-Adaptive GNNs}
We introduce a novel mechanism that leverages Bakry–Émery curvature to adaptively control the message-passing depth of graph neural networks (GNNs). In our framework, vertices with higher curvature (which typically indicate rapid local diffusion) terminate message passing earlier, while vertices with lower curvature continue aggregating messages to capture a broader neighborhood. The theoretical justification for this design choice is provided in Section~\ref{sec:theoretical_analysis}.

\subsection{Adaptive Layer Depth Mechansim}

Let \(t \in \mathbb{N}\) denote the iteration index in the message-passing process. To assign a stopping depth \(T(x) \in \mathbb{N}\) to each vertex \(x\), we rank the vertices based on their Bakry–Émery curvature \(\kappa(x)\). For a given threshold \(k\%\), we define
\[
T(x) \coloneqq \min\left\{ t \in \mathbb{N} \,\Big|\, \frac{1}{|V|}\sum_{y \in V} \mathbb{I}\Bigl(\kappa(y) \ge \kappa(x)\Bigr) \le \frac{k\,t}{100} \right\},
\]
where \( \mathbb{I}(\cdot) \) is the indicator function which is 1 if $\kappa(y) \ge \kappa(x)$ and 0 otherwise. This ensures that a vertex \(x\) is assigned a stopping depth \(T(x) = t\) when the proportion of vertices with curvature at least \(\kappa(x)\) falls below the threshold \(\frac{k\,t}{100}\). 

During message passing, for each vertex \(x\) and iteration \(t \le T(x)\), we update its embedding by aggregating messages from its neighbors. To incorporate the adaptive depth, the message from a neighbor \(y\) is taken from the layer \(\gamma = \min\{t-1,\, T(y)\}\). Formally, the update equations are:
\begin{align}
m_x^{(t)} &= \text{AGG}\Bigl(\{\!\!\{ h_y^{(\min\{t-1,\, T(y)\})} \mid y \in N(x) \}\!\!\}\Bigr), \label{eq:agg_depth}\\[1mm]
h_x^{(t)} &= \text{UPD}\Bigl( h_x^{(t)},\, m_x^{(t-1)} \Bigr). \label{eq:upd_depth}
\end{align}
After \(t = T(x)\) iterations, the final representation \(h_x^{(T(x))}\) is fed into a task-specific module (e.g., a classifier or regressor) \(\mathcal{C}\):
\[
\hat{y}_x = \mathcal{C}\bigl( h_x^{(T(x))} \bigr)
\]
where \( \hat{y}_x \) represents the predicted output for vertex \( x \).

This adaptive scheme is model-agnostic and can be integrated into existing GNN architectures such as GIN~\cite{xu2018powerful}, GCN~\cite{kipf2016semi}, GAT~\cite{velivckovic2018graph}, and GraphSAGE~\cite{hamilton2017inductive}. The high-level architecture of our approach is depicted in Figure \ref{fig:model summary}.









%% file: sections/optimization.tex
\begin{table*}[htbp]
\centering
\renewcommand\arraystretch{1.1}
\scalebox{0.85}{\begin{tabular}{c| c c c | c c c c c | c}
\specialrule{.1em}{.05em}{.05em} 
Methods & Cora & Citeseer & Pubmed & Actor & Squirrel & Cornell  &Wisconsin & Texas &  ogbn-arxiv\\ 
\toprule
{GCN} & {87.28 \sd{1.26}} & {76.68 \sd{1.64}} & {87.38 \sd{0.66}} & {30.26 \sd{0.79}}  & {36.89 \sd{1.34}} & {57.03 \sd{4.67}} & {59.80 \sd{6.99}} & {59.46 \sd{5.25}}\ & {71.74 \sd{0.29}} \\
{GCN+BEC} & {88.50 \sd{1.32}} & {80.43 \sd{0.88}} & {88.55 \sd{0.68}} & {31.66 \sd{0.58}}  & {37.82 \sd{0.60}} & {61.62 \sd{4.64}} & {69.41 \sd{2.85}} & {75.68 \sd{2.18}} & {71.91 \sd{0.17}} \\ 
{$\Delta \uparrow$} & {\textcolor{blue}{+1.22}} & {\textcolor{blue}{+3.75}} & {\textcolor{blue}{+1.17}} & {\textcolor{blue}{+1.40}}  & {\textcolor{blue}{+0.93}} & {\textcolor{blue}{+4.59}} & {\textcolor{blue}{+9.61}} & {\textcolor{blue}{+16.22}} & {\textcolor{blue}{+0.17}} \\ 
\midrule
{GAT} & {82.68 \sd{1.80}} & {75.46 \sd{1.72}} & {84.68 \sd{0.44}} & {26.28 \sd{1.73}}  & {30.62 \sd{2.11}} & {58.92 \sd{3.32}} & {55.29 \sd{8.71}} & {58.38 \sd{4.45}} & {71.54 \sd{0.30}} \\
{GAT+BEC} & {85.39 \sd{0.69}} & {76.12 \sd{3.23}} & {86.71 \sd{0.46}} & {31.22 \sd{0.56}}  & {31.61 \sd{0.47}} & {61.62 \sd{5.64}} & {63.23 \sd{3.28}} & {71.56 \sd{5.56}} & {71.84 \sd{0.31}} \\ 
{$\Delta \uparrow$} & {\textcolor{blue}{+2.71}} & {\textcolor{blue}{+0.66}} & {\textcolor{blue}{+2.03}} & {\textcolor{blue}{+4.94}}  & {\textcolor{blue}{+0.99}} & {\textcolor{blue}{+2.70}} & {\textcolor{blue}{+7.94}} & {\textcolor{blue}{+13.18}} & {\textcolor{blue}{+0.30}} \\ 
\midrule
{GraphSAGE} & {86.90 \sd{1.04}} & {76.04 \sd{1.30}} & {88.45 \sd{0.50}} & {34.23 \sd{0.99}}  & {41.61 \sd{0.74}} & {75.95 \sd{5.01}} & {81.18 \sd{5.56}} & {82.43 \sd{6.14}} & {71.49 \sd{0.27}} \\
{GraphSAGE+BEC} & {87.08 \sd{0.42}} & {78.01 \sd{1.33}} & {89.00 \sd{0.31}} & {35.56 \sd{0.87}}  & {44.53 \sd{0.81}} & {81.08 \sd{3.82}} & {88.97 \sd{1.88}} & {89.41 \sd{3.06}} & {71.83 \sd{0.14}} \\ 
{$\Delta \uparrow$} & {\textcolor{blue}{+0.18}} & {\textcolor{blue}{+1.97}} & {\textcolor{blue}{+0.55}} & {\textcolor{blue}{+1.33}}  & {\textcolor{blue}{+2.92}} & {\textcolor{blue}{+5.13}} & {\textcolor{blue}{+7.79}} & {\textcolor{blue}{+6.98}} & {\textcolor{blue}{+0.34}} \\ 
\midrule
{SGC} & {84.97 \sd{1.95}} & {75.66 \sd{1.37}} & { 87.15 \sd{0.47}} & { 25.83 \sd{1.09}}  & { 41.78 \sd{2.88}} & {55.41 \sd{5.29}} & {57.84 \sd{4.83}} & { 58.11 \sd{6.27}} & {68.74 \sd{0.12}} \\
{SGC+BEC} & {85.49 \sd{0.65}} & {78.32 \sd{0.66}} & {87.94 \sd{0.23}} & {27.97 \sd{2.18}}  & {42.03 \sd{0.62}} & {65.13 \sd{3.29}} & {63.97 \sd{9.30}} & {69.80 \sd{1.79}} & {71.39 \sd{0.30}} \\ 
{$\Delta \uparrow$} & {\textcolor{blue}{+0.52}} & {\textcolor{blue}{+2.66}} & {\textcolor{blue}{+0.79}} & {\textcolor{blue}{+2.14}}  & {\textcolor{blue}{+0.25}} & {\textcolor{blue}{+9.72}} & {\textcolor{blue}{+6.13}} & {\textcolor{blue}{+11.69}}\ & {\textcolor{blue}{+2.65}} \\ 
\midrule
{MixHop} & {87.61 \sd{0.85}} & {76.26 \sd{1.33}} & {85.31 \sd{0.61}} & {32.22 \sd{2.34}}  & {43.80 \sd{1.48}} & {73.51 \sd{6.34}} & {75.88 \sd{4.90}} & {77.84 \sd{7.73}} & {70.83 \sd{0.30}} \\
{MixHop+BEC} & {87.89 \sd{0.67}} & {76.79 \sd{0.95}} & {88.18 \sd{0.37}} & {36.91 \sd{0.93}}  & {43.97 \sd{0.67}} & {78.37 \sd{2.70}} & {93.82 \sd{1.10}} & {88.23 \sd{1.75}} & {72.04 \sd{0.26}} \\
{$\Delta \uparrow$} & {\textcolor{blue}{+0.28}} & {\textcolor{blue}{+0.53}} & {\textcolor{blue}{+2.87}} & {\textcolor{blue}{+4.69}}  & {\textcolor{blue}{+0.17}} & {\textcolor{blue}{+4.86}} & {\textcolor{blue}{+17.94}} & {\textcolor{blue}{+10.39}} & {\textcolor{blue}{+1.21}} \\
\midrule
{DIR-GNN} & {82.89 \sd{0.44}} & {76.55 \sd{0.72}} & {88.84 \sd{0.13}} & {30.12 \sd{0.65}}  & {42.50 \sd{0.80}} & {80.00 \sd{3.60}} & {79.60 \sd{3.80}} & {81.40 \sd{2.40}} & {70.72 \sd{0.26}} \\
{DIR-GNN+BEC} & {84.12 \sd{0.76}} & {77.67 \sd{1.08}} & {89.40 \sd{0.19}} & {32.46 \sd{0.43}}  & {46.19 \sd{1.26}} & {85.94 \sd{3.78}} & {89.55 \sd{1.22}} & {87.84 \sd{3.01}} & {71.55 \sd{0.26}} \\ 
{$\Delta \uparrow$} & {\textcolor{blue}{+1.23}} & {\textcolor{blue}{+1.12}} & {\textcolor{blue}{+0.56}} & {\textcolor{blue}{+2.34}} &{\textcolor{blue}{+3.69}} & {\textcolor{blue}{+5.94}} & {\textcolor{blue}{+9.95}} & {\textcolor{blue}{+6.44}} & {\textcolor{blue}{+0.83}} \\ 
\bottomrule
\end{tabular}}
\caption{
Vertex classification accuracy ± std (\%). The baseline results are sourced from \cite{zhu2020beyond,sun2024breaking,brodyattentive,han2023alternately}.}
\label{Tab:node-classification-baselines}

\end{table*}

\subsection{Learning to Estimate Curvature}
Accurate estimation of \(\kappa(x)\) is essential for determining the stopping depth \(T(x)\). We cast curvature estimation as an optimization problem that enforces the Bakry–Émery curvature-dimension inequality. Let $\hat{\kappa}(x)$ denote the estimated curvature of vertex $x$. For any function \(f: V \to \mathbb{R}\), we define the penalty
\[
\mathcal{L}(\hat{\kappa}(x), f) \coloneqq \max\bigl\{ 0,\; \hat{\kappa}(x)\,\Gamma(f,f)(x) - \Gamma_2(f,f)(x) \bigr\}.
\]
This penalty quantifies the violation of the inequality
\[
\Gamma_2(f,f)(x) \ge \hat{\kappa}(x)\,\Gamma(f,f)(x).
\]
When the inequality holds at vertex \(x\) for \(f\), we have \(\mathcal{L}(\hat{\kappa}(x), f)=0\); otherwise, the penalty is positive. Ideally, the true curvature \(\kappa_{\text{true}}(x)\) is the largest value such that
\[
\mathcal{L}(\kappa_{\text{true}}(x), f) = 0 \quad \text{for all } f: V \to \mathbb{R}.
\]

Since optimizing over the entire function space is intractable, we restrict our search to a finite, parameterized set \(\mathcal{F}\). In practice, we employ a Multi-Layer Perceptron (MLP) as a flexible function approximator. Let $
f_\theta: V \to \mathbb{R}$ 
denote the function represented by an MLP with parameters \(\theta\). By sampling parameter configurations \(\{\theta_i\}_{i=1}^N\), we obtain a candidate set
\[
\mathcal{F} = \{ f_{\theta_1}, f_{\theta_2}, \dots, f_{\theta_N} \}.
\]
We then approximate the estimated curvature by
\[
\hat{\kappa}(x) = \sup_{\kappa \in \mathbb{R}} \inf_{f \in \mathcal{F}} \mathcal{L}(\kappa, f).
\]
This optimization process seeks the maximum \(\hat{\kappa}(x)\) such that the curvature-dimension inequality is ``nearly'' satisfied for all functions \(f \in \mathcal{F}\).

We highlight two key properties of our curvature estimation approach:

\begin{enumerate}
    \item \textbf{Upper bound:} Since \(\mathcal{F}\) is a proper subset of the space of all functions \(f: V \to \mathbb{R}\), the restricted optimization can only overestimate the true curvature. Formally, if 
    \[
    \hspace*{0.8cm}\kappa_{\text{true}}(x) = \sup\bigl\{ \kappa \in \mathbb{R} : \mathcal{L}(\kappa, f)=0 \text{ for all } f: V \to \mathbb{R} \bigr\},
    \]
    then  
    \[
    \hat{\kappa}(x) = \sup_{\kappa \in \mathbb{R}} \inf_{f \in \mathcal{F}} \mathcal{L}(\kappa, f) \ge \kappa_{\text{true}}(x).
    \]
    
    \item \textbf{Smoothness:} In order for the differential operators \(\Gamma\) and \(\Gamma_2\) to be well-defined, the candidate functions in \(\mathcal{F}\) must be smooth~\cite{bakry2006diffusions}. However, standard MLPs do not inherently guarantee smooth approximations, as many common activation functions (e.g., ReLU) lack smoothness. To ensure smooth outputs, we employ an MLP with \(C^\infty\) activation functions (e.g., the sigmoid function in our implementation), thereby guaranteeing that each function \(f_\theta \in \mathcal{F}\) is infinitely differentiable.\looseness=-1
\end{enumerate}

This approach leverages the expressive power of neural networks to approximate a rich subset of functions, facilitating a robust, task-specific estimation of vertex curvature.

\subsection{GNN Training Loss Function}
We jointly optimize the GNN for the downstream task and for compliance with the curvature-dimension inequality by combining two loss terms: the task loss and the curvature loss.

Let \(\mathcal{L}_{\text{task}}\) denote the loss for the downstream task (e.g., cross-entropy for classification or mean squared error for regression). To enforce the curvature-dimension inequality, we define the curvature loss as
\[
\mathcal{L}_{\text{curv}} = \sum_{x \in V} \left( \sum_{f \in \mathcal{F}} \mathcal{L}\bigl(\hat{\kappa}(x), f\bigr) - \lambda\, \hat{\kappa}(x) \right).
\]
Here, $\lambda$ is a regularization constant controlling the tradeoff between maximizing $\hat{\kappa}(x)$ and ensuring the curvature-dimension inequality. The overall training loss is then defined as
\[
\mathcal{L}_{\text{total}} = \mathcal{L}_{\text{task}} + \mathcal{L}_{\text{curv}}.
\]
optimizes GNN performance while enforcing adherence to the graph's geometric structure via curvature.

%% file: sections/theoretical_analysis.tex
\begin{table*}[htbp]
    \centering
    \begin{minipage}
    {0.9\textwidth}
    \centering
    \begin{adjustbox}{width=0.95\textwidth}
    \begin{tabular}{l|cc|cc|cc|cc}
        \toprule
       \multirow{2}{*}{Methods} & \multicolumn{2}{c}{Jazz} & \multicolumn{2}{c}{Cora-ML} & \multicolumn{2}{c}{Network Science} & \multicolumn{2}{c}{Power Grid}\\
        \cmidrule(lr){2-3} \cmidrule(lr){4-5} \cmidrule(lr){6-7} \cmidrule(lr){8-9} 
         & IC & LT & IC & LT & IC & LT & IC & LT\\
        \midrule
        GCN & 0.233 \sd{0.010} & 0.199 \sd{0.006} & 0.277 \sd{0.007} & 0.255 \sd{0.008} & 0.270 \sd{0.019} & 0.190 \sd{0.012} & 0.313 \sd{0.024} & 0.335 \sd{0.023}  \\
        GCN+BEC & 0.202 \sd{0.007} & 0.124 \sd{0.002} & 0.258 \sd{0.006} & 0.194 \sd{0.010} & 0.233 \sd{0.010} & 0.123 \sd{0.026} &  0.329 \sd{0.019} & 0.298 \sd{0.039} \\
        $\Delta \uparrow$ & \textcolor{blue}{+0.031} & \textcolor{blue}{+0.075} &  \textcolor{blue}{+0.019} & \textcolor{blue}{+0.061} & \textcolor{blue}{+0.037} & \textcolor{blue}{+0.067} & {-0.016} & \textcolor{blue}{+0.037} \\
        \midrule
        GAT & 0.342 \sd{0.005} & 0.156 \sd{0.100} & 0.352 \sd{0.004} & 0.192 \sd{0.010} & 0.274 \sd{0.002} & 0.114 \sd{0.008} & 0.331 \sd{0.002} & 0.280 \sd{0.015}   \\
        GAT+BEC & 0.238 \sd{0.035} & 0.123 \sd{0.002} & 0.269 \sd{0.014} & 0.184 \sd{0.015} & 0.233 \sd{0.016} & 0.126 \sd{0.010} & 0.291 \sd{0.016} & 0.277 \sd{0.013} \\
        $\Delta \uparrow$ & \textcolor{blue}{+0.104} & \textcolor{blue}{+0.033} & \textcolor{blue}{+0.083} & \textcolor{blue}{+0.008} & \textcolor{blue}{+0.041} & {-0.012} & \textcolor{blue}{+0.040} & \textcolor{blue}{+0.003} \\
        \midrule
        GraphSAGE & 0.201 \sd{0.028}  & 0.120 \sd{0.004} & 0.255 \sd{0.010} & 0.203 \sd{0.019} & 0.241 \sd{0.010} & 0.112 \sd{0.005} & 0.313 \sd{0.024} & 0.341 \sd{0.018} \\
        GraphSAGE+BEC &  0.190 \sd{0.015} & 0.060 \sd{0.010} & 0.246 \sd{0.008} & 0.176 \sd{0.008} & 0.231 \sd{0.010} & 0.073 \sd{0.012} & 0.302 \sd{0.028} & 0.205 \sd{0.013} \\
        $\Delta \uparrow$ & \textcolor{blue}{+0.011} & \textcolor{blue}{+0.060} &  \textcolor{blue}{+0.009} & \textcolor{blue}{+0.027} & \textcolor{blue}{+0.010} & \textcolor{blue}{+0.039}  & \textcolor{blue}{+0.011} & \textcolor{blue}{+0.136}\\
        \bottomrule
    \end{tabular}
    \end{adjustbox}
    \label{tab:combined_results}
    \caption{Influence estimation performance MAE ± std  under IC and LT diffusion models.}\vspace{-0.3cm}
    \label{tab:im_results}
    \end{minipage}
\end{table*}





\section{Theoretical Analysis}
\label{sec:theoretical_analysis}
To gain deeper insight into the computational and representational properties of our method, we first analyze its complexity, then explore the connection between feature distinctiveness and curvature.
 \subsection{Complexity Analysis}

We first analyze the time and space complexity of our optimization process. For each vertex \(x \in V\), let \(d_x\) denote its degree, and let \(d_{\text{max}}\) denote the maximum vertex degree in the graph. The computational cost of each operation depends on the local structure of the graph, particularly the degree distribution.

The time complexity of Equation \eqref{eq1}, the weighted Laplacian \(\Delta f(x)\), and Equation \eqref{eq2}, the squared gradient operator \(\Gamma(f, f)(x)\), is \(O(d_x)\). In contrast, computing the convexity operator \(\Gamma_2(f, f)(x)\) (Equation \eqref{eq3}) requires \(O(d_x^2)\) operations per vertex. Thus, the overall computational cost over all vertices is \(\mathcal{O}(\sum_{x \in V} d_x^2)\), which can be upper-bounded by \(O(|V| \cdot d_{\text{max}}^2)\). In our approach, the optimization is performed over a finite set \(\mathcal{F}\) of candidate functions, with \(|\mathcal{F}| = N\). Consequently, the total time complexity becomes
\[
O\bigl(N \cdot |V| \cdot d_{\text{max}}^2\bigr).
\]
Since real-world graphs are typically sparse (i.e., \(d_{\text{max}} \ll |V|\)) and \(N\) is small (in our experiments, \(N \le 5\)), the overall time complexity is manageable. The space complexity is \(\mathcal{O}(|V| + |E|)\), accounting for the storage of vertex and edge data.

We further analyze empirical runtime and parameter complexity in \Cref{subsec: ablantion-studies}.



\subsection{Feature Distinctiveness and Curvature}

We establish a connection between a vertex's feature distinctiveness in a GNN and its Bakry–Émery curvature. Here, feature distinctiveness is quantified by the variation in vertex features across layers. In particular, we show that the rate at which feature distinctiveness decays is governed by the Bakry–Émery curvature, reflecting how quickly information diffuses over the graph. This analysis assumes that vertex features evolve approximately according to heat flow driven by the graph Laplacian.

\begin{restatable}[Feature Decay Bound]{theorem}{thmfd}\label{thmfd}
Let \( G = (V, E, w) \) be an undirected, weighted graph with bounded degree. Consider a GNN with \( l \) layers, where each layer approximates the heat flow for a time step \( \Delta t \). For a vertex \( x \in V \) with Bakry-Émery curvature \( \kappa(x) \), we define the feature distinctiveness \( D(x,l) \) after \( l \) layers as
\[
D(x,l) = \frac{|\nabla f_l|^2(x)}{|\nabla f_0|^2(x)},
\]
where \( f_l \) represents the vertex features at layer \( l \). Then, the following holds: 
\[
D(x, l) \leq e^{-\kappa(x) l \Delta t}.
\]
For any \( \epsilon > 0 \), to maintain \( D(x,l) \geq \epsilon \), it suffices that the number of layers satisfy:
\[
l \leq \frac{\log(1/\epsilon)}{\kappa(x) \Delta t}.
\]

\end{restatable}

In a nutshell, the higher curvature vertices, with stronger local connectivity, lose feature distinctiveness more quickly in a GNN due to faster information propagation. Therefore, fewer layers are required to maintain a given level of distinctiveness \( \epsilon \) for high curvature vertices. In contrast, low curvature vertices have weaker local connections, causing feature distinctiveness to decay more slowly. To preserve \( \epsilon \) for low curvature vertices, more layers are needed. Thus, the number of layers required to maintain feature distinctiveness tends to decrease with increasing vertex curvature.

%% file: sections/experiments.tex
\begin{table*}[t!]
\scalebox{0.82}{\begin{tabular}{c|ccccccc|c} 
\toprule
{Methods}& {MUTAG} & {PTC-MR}  & {IMDB-BINARY} & {COX2} & {BZR}  & {PROTEINS} & {ogbg-moltox21} & {ZINC} \\ 
\toprule
{GIN} & {92.8 \sd{5.9}}  & {65.6 \sd{6.5}}  & {78.1 \sd{3.5}} & {88.9 \sd{2.3}}  & {91.1 \sd{3.4}}  & {78.8 \sd{4.1}} & {74.9 \sd{0.5}} & {0.387 \sd{0.015}}\\
{GIN+BEC} & {96.1 \sd{3.6}}  & {72.9 \sd{5.7}}  & {80.8 \sd{3.3}}  & {89.3 \sd{3.1}}  & {92.4 \sd{3.6}}  & {79.1 \sd{3.7}} & {75.7 \sd{0.7}} & {0.308 \sd{0.009}}\\
{$\Delta \uparrow$ } & {\textcolor{blue}{+3.3}}  & {\textcolor{blue}{+7.3}}  & {\textcolor{blue}{+2.7}}  & {\textcolor{blue}{+0.4}} & {\textcolor{blue}{+1.3}}  & {\textcolor{blue}{+0.3}} & {\textcolor{blue}{+0.8}}  & {\textcolor{blue}{+0.079}}\\
\midrule
{GCN} & {92.2 \sd{4.4}}  & {68.8 \sd{6.2}}  & {79.8 \sd{2.3}} & {88.5 \sd{3.8}}   & {92.6 \sd{4.8}}  & {78.8 \sd{3.9}} & {75.3 \sd{0.7}} &{0.459 \sd{0.006}} \\
{GCN+BEC} & {95.0 \sd{3.0}}  & {70.6 \sd{5.1}}  & {79.9 \sd{3.4}}  & {89.3 \sd{2.8}}  & {93.6 \sd{3.8}}  & {79.6 \sd{4.4}} & {75.5 \sd{0.5}} & {0.424 \sd{0.049}}\\
{$\Delta \uparrow$ } & {\textcolor{blue}{+2.8}}  & {\textcolor{blue}{+1.8}}  & {\textcolor{blue}{+0.1}}  & {\textcolor{blue}{+0.8}}  & {\textcolor{blue}{+1.0}}  & {\textcolor{blue}{+0.8}} & {\textcolor{blue}{+0.2}} & {\textcolor{blue}{+0.035}}\\
\midrule
{ID-GNN} & {97.8 \sd{3.7}}  & {74.4 \sd{4.2}}  & {79.3 \sd{2.9}} & {87.8 \sd{3.1}}  & {93.3 \sd{4.8}}  & {78.1 \sd{3.9}} & {74.7 \sd{0.5}} & {0.366 \sd{0.014}}\\
{ID-GNN+BEC} & {98.3 \sd{3.6}}  & {75.6 \sd{4.0}}  & {81.5 \sd{2.4}}  & {89.3 \sd{3.1}}  & {93.8 \sd{4.5}}  & {78.6 \sd{3.8}} & {75.4 \sd{0.7}} & {0.350 \sd{0.013}}\\
{$\Delta \uparrow$ } & {\textcolor{blue}{+0.5}}  & {\textcolor{blue}{+0.8}}  & {\textcolor{blue}{+2.2}}  & {\textcolor{blue}{+1.5}}  & {\textcolor{blue}{+0.5}}  & {\textcolor{blue}{+0.5}} & {\textcolor{blue}{+0.7}} & {\textcolor{blue}{+0.016}}\\
\midrule
{GraphSNN} & {94.7 \sd{1.9}}  & {70.6 \sd{3.1}}  & {78.5 \sd{2.3}} & {86.3 \sd{3.3}}  & {91.1 \sd{3.0}}  & {78.4 \sd{2.7}} & {75.5 \sd{1.1}} & {0.297 \sd{0.004}}\\
{GraphSNN+BEC} & {96.1 \sd{2.5}}  & {72.9 \sd{7.4}}  & {79.4 \sd{2.7}}  & {88.9 \sd{2.0}}  & {92.1 \sd{4.9}}  & {79.2 \sd{3.9}} & {75.3 \sd{0.7}} & {0.265 \sd{0.005}}\\
{$\Delta \uparrow$ } & {\textcolor{blue}{+1.4}}  & {\textcolor{blue}{+2.3}}  & {\textcolor{blue}{+0.9}}  & {\textcolor{blue}{+2.6}}  & {\textcolor{blue}{+1.0}}  & {\textcolor{blue}{+0.8}} & {-0.2} & {\textcolor{blue}{+0.032}}\\
\midrule
{NC-GNN} & {92.8 \sd{5.0}}  & {71.8 \sd{6.2}}  & {78.4 \sd{4.0}} & {88.4 \sd{3.3}}  & {92.6 \sd{4.3}}  & {78.4 \sd{3.1}} & {75.1 \sd{0.4}} & {0.448 \sd{0.095}}\\
{NC-GNN+BEC} & {96.0 \sd{2.4}}  & {75.0 \sd{4.6}}  & {79.6 \sd{3.6}}  & {88.9 \sd{2.4}}  & {93.6 \sd{4.9}}  & {79.6 \sd{2.9}} & {75.6 \sd{0.5}} & {0.356 \sd{0.017}}\\
{$\Delta \uparrow$ } & {\textcolor{blue}{+3.2}}  & {\textcolor{blue}{+3.2}}  & {\textcolor{blue}{+1.2}}  & {\textcolor{blue}{+0.5}}  & {\textcolor{blue}{+1.0}}  & {\textcolor{blue}{+1.2}} & {\textcolor{blue}{+0.5}} & {\textcolor{blue}{+0.092}}\\
\bottomrule
\end{tabular}}
\caption{Graph classification accuracy ± std (\%) for TU datasets, ROC ± std (\%) for the OGB dataset, and graph regression MAE ± std for ZINC dataset, with baseline results sourced from \citet{wijesinghe2022new}, \citet{feng2022powerful}, and \citet{hu2020open}.}
\label{Tab:graph-classification-baselines}
\end{table*}

\section{Experiments}

We conduct experiments on four widely used benchmark tasks: vertex classification, vertex regression, graph classification, and graph regression, utilizing a total of 21 datasets and comparing against 10 baselines to evaluate the effectiveness of our approach.

\subsection{ Experimental Setup}

\subsubsection{Datasets}

For vertex classification, we use three widely adopted graphs exhibiting strong homophily: Cora, PubMed, and Citeseer \cite{sen2008collective,namata2012query}, as well as five datasets with heterophily: Texas, Wisconsin, Actor, Squirrel, and Cornell \cite{rozemberczki2021multi,tang2009social}. Additionally, we include a large-scale dataset from the Open Graph Benchmark, ogbn-arxiv \cite{hu2020open}. For vertex regression, we utilize four real-world datasets (Jazz, Network Science, Cora-ML, and Power Grid) \cite{rossi2015network,mccallum2000automating}. For graph classification, we use six small-scale real-world datasets from TU Datasets (MUTAG, PTC-MR, COX2, BZR, PROTEINS, and IMDB-BINARY) \cite{morris2020tudataset}, as well as a large-scale molecular dataset from the Open Graph Benchmark, ogbg-moltox21 \cite{hu2020open}. Finally, we employ the ZINC 12k dataset \cite{dwivedi2023benchmarking} for graph regression.

\subsubsection{Baselines}

For vertex classification and regression baselines, we use three classical GNNs—GCN \cite{kipf2016semi}, GAT \cite{velivckovic2018graph}, and GraphSAGE \cite{hamilton2017inductive}—and for vertex classification, we further employ three enhanced GNNs—SGC \cite{wu2019simplifying}, MixHop \cite{abu2019mixhop}, and DIR-GNN \cite{rossi2024edge}.  For graph classification and regression baselines, we select two GNNs whose expressive power is known to be upper-bounded by the 1-dimensional Weisfeiler-Lehman (1-WL) algorithm \cite{zhang2024expressive}: GIN \cite{xu2018powerful} and GCN \cite{kipf2016semi}. Additionally, we include three GNNs designed to possess expressive power beyond 1-WL: ID-GNN \cite{you2021identity}, GraphSNN \cite{wijesinghe2022new}, and NC-GNN \cite{liuempowering}.

\subsubsection{Evaluation settings}

For vertex classification, we use the feature vectors and class labels with 10 random splits for all datasets except the OGB dataset. In these splits, 48\% of the vertices are assigned for training, 32\% for validation, and 20\% for testing, following \citet{peigeom}. For the OGB dataset, we adopt the setup described by \citet{hu2020open}. For the vertex regression task, we focus on influence estimation \cite{xia2021deepis,ling2023deep}, a key problem in network science, where the goal is to predict each vertex's influence based on graph structure and diffusion dynamics. We consider the Linear Threshold (LT) \cite{granovetter1978threshold} and Independent Cascade (IC) \cite{goldenberg2001talk} models, which capture different information spread mechanisms, with an initial activation set comprising 10\% of vertices. Following \citet{ling2023deep}, we use their data splits and adopt 10-fold cross-validation, as in \citet{errica2020fair}. In graph classification, we adopt the setup outlined by \citet{hu2020open} for the OGB dataset. For the TU datasets, we follow the experimental setup provided by \citet{feng2022powerful}, reporting the mean and standard deviation of the best accuracy across 10 test folds. For graph regression, we adhere to the setup provided by \citet{dwivedi2023benchmarking}. We report baseline results from previous works using the same experimental setup, where available. If such results are unavailable, we generate baseline results by adopting the hyperparameter configurations specified in the original papers.

Since the benchmark graphs are unweighted and the Bakry-Émery curvature operators (Equations \eqref{eq1}, \eqref{eq2}, and \eqref{eq3}) require edge weights, we employ learnable parameters for these weights. Additionally, the functions used to evaluate the Bakry-Émery curvature inequality take vertex features as inputs. 

Comprehensive details regarding downstream tasks, baselines, dataset statistics, and model hyperparameters can be found in the Appendix.

\subsection{Main Results}

We use GNN+BEC to represent the GNN model incorporating our Bakry-Émery curvature-based adaptive layer depth mechanism, and $\Delta \uparrow$ indicates the performance improvement of our approach compared to the baseline.

\subsubsection{Vertex Classification}

Table~\ref{Tab:node-classification-baselines} presents the results for the vertex classification task. Our approach consistently outperforms the baselines in both homophilic and heterophilic settings. Notably, our adaptive layer mechanism achieves greater performance improvements on heterophilic graphs compared to homophilic ones. We attribute this to the ability of our method to learn data-driven curvature values that leverage heterophilic label patterns, effectively capturing the complex and diverse relationships between vertices with dissimilar labels, which leads to enhanced aggregation and feature extraction. Additionally, our approach improves the performance of GNNs designed with inductive biases for heterophilic graph settings, such as MixHop, and DIR-GNN. This demonstrates that our method complements these models by capturing additional structural and label-dependent information that might otherwise be overlooked.

We also observe that our depth-adaptive mechanism significantly reduces the standard deviation of accuracy across baseline models, enhancing stability and consistency. This reduced variance ensures reliable and repeatable results, essential for real-world applications.\looseness=-1

\subsubsection{Vertex Regression}

Table \ref{tab:im_results} presents the results for the influence estimation task, formulated as a vertex regression problem. The results indicate that our approach significantly enhances baseline performance in predicting influence by effectively adapting to various underlying diffusion processes. By dynamically adjusting curvature based on the propagation mechanisms of different diffusion models, our method enhances generalization of the baseline GNN models across diverse network conditions.

\subsubsection{Graph Classification} We present graph classification results in Table \ref{Tab:graph-classification-baselines}. These results show that our adaptive layer mechanism achieves notable to moderate improvements over baseline methods across all datasets. Significant gains are observed with the GIN model, which is commonly regarded as the standard baseline for graph classification tasks. Moreover, our approach enhances GNNs with higher expressive power than 1-WL by improving feature distinctiveness, thus offering a more effective mechanism for distinguishing between graph structures. 

\subsubsection{Graph Regression} The regression results for ZINC 12K dataset is depicted in Table \ref{Tab:graph-classification-baselines}. Similar to the graph classification, the proposed adaptive layer mechanism consistently improves performance over baseline methods. These results underscore the effectiveness of our approach in enhancing feature representation and regularization, thereby improving the overall performance of graph regression tasks. Although vertex-specific, our curvature provides insights into the global graph structure, enhancing performance in tasks that depend on holistic properties of graphs.

\subsection{Ablation Studies}\label{subsec: ablantion-studies}

In this section, we conduct experimental analysis to assess different aspects of our work.

\subsubsection{Hyper-Parameter Analysis}

In this experiment, we analyze the model's sensitivity to its two primary hyper-parameters. Using GraphSAGE as the base model, we evaluate vertex classification performance on two datasets. Figure \ref{fig:hyperparameter_analysis} (a) illustrates how performance varies with the number of functions used to approximate curvature. As expected, increasing the number of functions enhances performance by enabling the model to capture multiple aspects of information propagation. However, this improvement is only observed up to a certain threshold, beyond which the additional complexity—introduced by an increased number of learnable parameters—begins to outweigh the benefits, ultimately leading to performance degradation. Notably, the model achieves good performance with a relatively small number of functions, which is advantageous from a computational efficiency perspective. 

\begin{figure}[H]
    \centering
    \begin{tikzpicture}
        \begin{axis}[
            hide axis,
            axis lines=none,
            scale only axis,
            height=0pt,
            width=0pt,
            legend style={
                font=\small,  
                legend cell align=left,
                legend columns=2,  
                column sep=1em,    
                draw=none,
                at={(0,0)},
                anchor=center
            },
            legend image post style={scale=0.8}  
        ]
            \addplot[color=blue, mark=square, mark size=3] coordinates {(0,0)};
            \addlegendentry{Citeseer}
            \addplot[color=red, mark=diamond, mark size=4] coordinates {(1,0)};
            \addlegendentry{Wisconsin}
        \end{axis}
    \end{tikzpicture}
    
    \vspace{-0.7em}  
    
    \begin{minipage}[b]{0.49\linewidth}
        \centering
        \subcaption{}
        \resizebox{0.9\linewidth}{!}{
        \begin{tikzpicture}
            \begin{axis}[    
                xlabel={\LARGE Number of Functions (N)},    
                ylabel={\LARGE Accuracy(\%)},
                xlabel style={font=\normalsize, font=\bfseries},
                ylabel style={font=\normalsize, font=\bfseries},
                xmin=1, xmax=16,    
                ymin=70, ymax=100,    
                xtick={1 ,2, 3, 5, 8, 12, 16},    
                ytick={70, 85, 100},    
                ymajorgrids=true,    
                grid style=dashed,
                legend style={legend pos=none}
            ]
            \addplot[color=blue, mark=square, mark size=4] coordinates {
                (1, 76.00)(2, 76.79) (3, 78.01)(5, 77.89)(8, 78.14)(12, 78.99)(16, 76.7)
            };
            \addplot[color=red, mark=diamond, mark size=5] coordinates {
                (1, 87.05)(2, 87.05) (3, 88.97)(5, 89.55)(8, 86.76)(12, 85.58)(16, 85.82)
            };
            \end{axis}
        \end{tikzpicture}
        }
    \end{minipage}\hfill
    \begin{minipage}[b]{0.49\linewidth}
        \centering
        \subcaption{}
        \resizebox{0.9\linewidth}{!}{
        \begin{tikzpicture}
            \begin{axis}[    
                xlabel={\LARGE Threshold (k\%)},    
                ylabel={\LARGE Accuracy(\%)},
                xlabel style={font=\normalsize, font=\bfseries},
                ylabel style={font=\normalsize, font=\bfseries},
                xmin=5, xmax=30,    
                ymin=70, ymax=100,    
                xtick={5, 10,15,20,25, 30},    
                ytick={70, 85, 100},    
                ymajorgrids=true,    
                grid style=dashed,
                legend style={legend pos=none}
            ]
            \addplot[color=blue, mark=square, mark size=4] coordinates {
                (5, 77.11)(10, 78.01)(15, 76.98)(20, 76.07)(25, 76.00)(30, 75.98)
            };
            \addplot[color=red, mark=diamond, mark size=5] coordinates {
                (5, 87.02)(10, 88.08)(15, 88.05)(20, 89.55)(25, 88.97)(30, 89.55)
            };
            \end{axis}
        \end{tikzpicture}
        }
        
    \end{minipage}\vspace{-0.2cm}
    \caption{Model sensitivity to hyper-parameters}
    \label{fig:hyperparameter_analysis}
\end{figure}
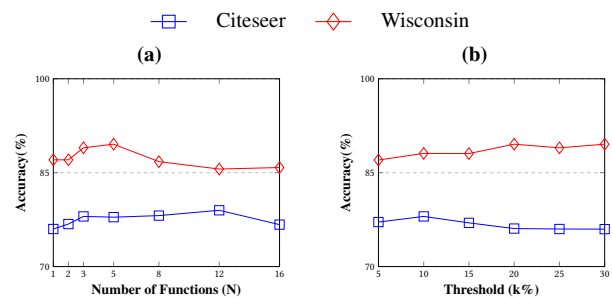

Figure \ref{fig:hyperparameter_analysis} (b) illustrates how model performance varies with the threshold. For Citeseer, optimal performance is achieved at a lower threshold (5–10\%), whereas Wisconsin benefits from a higher threshold. This difference stems from dataset characteristics: Citeseer, being homophilic, has adjacent vertices with the same labels, allowing deeper aggregation to enhance information mixing. In contrast, Wisconsin, a heterophilic dataset, benefits from shallow depth aggregation to prevent mixing information from differently labeled neighboring vertices.

\subsubsection{Oversmoothing Analysis}

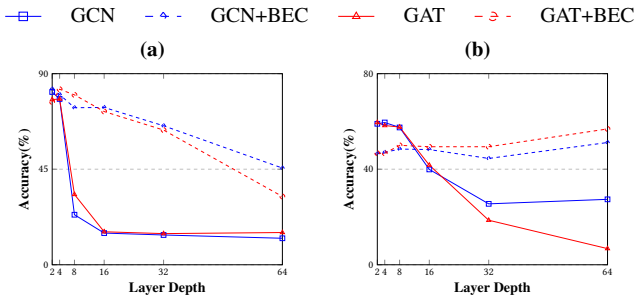
\begin{figure}[H]
    \centering
    
    \begin{tikzpicture}
        \begin{axis}[
            hide axis,
            scale only axis,
            height=0pt,
            width=0pt,
            legend style={
                font=\small,  
                legend cell align=left,
                legend columns=4,  
                column sep=1em,    
                draw=none,
                at={(0.5,1.05)},
                anchor=center
            },
            legend image post style={scale=0.8}  
        ]
            \addplot[color=blue, mark=square] coordinates {(0,0)};
            \addlegendentry{GCN}
            \addplot[color=blue, mark=diamond, style=dashed] coordinates {(0,0)};
            \addlegendentry{GCN+BEC}
            \addplot[color=red, mark=triangle] coordinates {(0,0)};
            \addlegendentry{GAT}
            \addplot[color=red, mark=o, style=dashed] coordinates {(0,0)};
            \addlegendentry{GAT+BEC}
        \end{axis}
    \end{tikzpicture}
    
    \vspace{-0.7em}  
    
    \begin{minipage}[b]{0.49\linewidth}
        \centering
        \subcaption{}
        \resizebox{0.9\linewidth}{!}{
        \begin{tikzpicture}
            \begin{axis}[    
                xlabel={\LARGE Layer Depth},    
                ylabel={\LARGE Accuracy(\%)},
                xlabel style={font=\normalsize, font=\bfseries},
                ylabel style={font=\normalsize, font=\bfseries},
                xmin=2, xmax=64,    
                ymin=0, ymax=90,    
                xtick={2, 4,8,16,32,64},    
                ytick={0, 45, 90},    
                ymajorgrids=true,    
                grid style=dashed,
                legend style={legend pos=none}
            ]
            \addplot[color=blue, mark=square] coordinates {
                (2, 81.38)(4, 78.09)(8, 23.56)(16, 14.92)(32, 14.02)(64, 12.48)
            };
            \addplot[color=blue, mark=diamond, style=dashed] coordinates {
                (2, 82.77)(4, 80.18)(8, 73.96)(16, 74.09)(32, 65.46)(64, 45.54)
            };
            \addplot[color=red, mark=triangle] coordinates {
                (2, 77.79)(4, 78.06)(8,  33.11)(16,  15.53)(32,  14.67)(64,  15.18)
            };
            \addplot[color=red, mark=o, style=dashed] coordinates {
                (2, 76.62)(4, 82.67)(8, 80.01)(16, 72.16)(32, 63.44)(64, 32.14)
            };
            \end{axis}
        \end{tikzpicture}
        }
    \end{minipage}\hfill
    \begin{minipage}[b]{0.49\linewidth}
        \centering
        \subcaption{}
        \resizebox{0.9\linewidth}{!}{
        \begin{tikzpicture}
            \begin{axis}[    
                xlabel={\LARGE Layer Depth},    
                ylabel={\LARGE Accuracy(\%)},
                xlabel style={font=\normalsize, font=\bfseries},
                ylabel style={font=\normalsize, font=\bfseries},
                xmin=2, xmax=64,    
                ymin=0, ymax=80,    
                xtick={2, 4,8,16,32,64},    
                ytick={0, 40, 80},    
                ymajorgrids=true,    
                grid style=dashed,
                legend style={legend pos=none}
            ]
            \addplot[color=blue, mark=square] coordinates {
                (2, 59.01)(4, 59.55)(8, 57.48)(16, 39.91)(32, 25.50)(64, 27.39)
            };
            \addplot[color=blue, mark=diamond, style=dashed] coordinates {
                (2, 46.47)(4, 47.05)(8, 48.43)(16, 48.23)(32, 44.50)(64, 51.17)
            };
            \addplot[color=red, mark=triangle] coordinates {
                (2, 59.54)(4, 58.28)(8,  57.65)(16,  41.71)(32, 18.64)(64,  06.76)
            };
            \addplot[color=red, mark=o, style=dashed] coordinates {
                (2, 46.66)(4, 46.66)(8, 50.00)(16, 49.41)(32, 49.41)(64, 56.86)
            };
            \end{axis}
        \end{tikzpicture}
        }
        
    \end{minipage}\vspace{-0.2cm}
    \caption{Oversmoothing comparison: (a) Cora; (b) Texas.}
    \label{fig:oversmoothing_analysis}
\end{figure}

We analyze the impact of our adaptive layer depth approach on oversmoothing in the vertex classification task by progressively increasing the layer depth up to 64 while measuring classification accuracy. Notably, we set $k$ = 1 to allow the aggregation stopping mechanism to remain effective even at large depths. As shown in Figure \ref{fig:oversmoothing_analysis}, our method exhibits substantial robustness against oversmoothing, achieving a significant performance margin over baseline models across both datasets.

\subsubsection{Comparison with Existing Curvature Notions}

\begin{table}[H]
\resizebox{1\columnwidth}{!}{
\begin{tabular}{c|cc|c c} 
\toprule
\multirow{3}{*}{Methods} & \multicolumn{2}{c|}{Vertex Classification} & \multicolumn{2}{c}{Vertex Regression} \\ 
\cmidrule(lr){2-5} 
 & {\small \multirow{2}{*}{Citeseer}} & {\small \multirow{2}{*}{Cornell}} & \multicolumn{2}{c}{\small Network Science }  \\ 
\cmidrule(lr){4-5} 
 &  &  & {\small IC} & {\small LT}  \\ 
\toprule
{GCN} & {76.68 \sd{1.64}}  & {57.03 \sd{4.67}}  & {0.270 \sd{0.019}} & {0.190 \sd{0.012}}  \\
\midrule
{GCN+ORC} & {78.49 \sd{0.35}}  & {57.83 \sd{4.22}}  & {0.251 \sd{0.006}} & {0.170 \sd{0.023}}  \\
{GCN+FRC} & {78.51 \sd{0.47}}  & {58.92 \sd{4.80}}  & {0.253 \sd{0.005}} & {0.178 \sd{0.024}}  \\
{GCN+JLC} & {77.03 \sd{0.56}}  & {58.64 \sd{3.20}}  & {0.239 \sd{0.010}} & {0.155 \sd{0.011}}  \\
{GCN+ERC} & {78.70 \sd{0.66}}  & {59.72 \sd{2.54}}  & {0.249 \sd{0.007}} & {0.155 \sd{0.029}}  \\
\midrule
{GCN+BC} & {78.10 \sd{0.26}}  & {45.68 \sd{3.51}}  & {0.248 \sd{0.005}} & {0.147 \sd{0.013}}  \\
{GCN+Degree} & {77.87 \sd{0.45}}  & {44.59 \sd{2.76}}  & {0.243 \sd{0.006}} & {0.159 \sd{0.019}}  \\
\midrule
{GCN+BEC} & {\textbf{80.43 \sd{0.88}}}  & {\textbf{61.62 \sd{4.64}}}  & {\textbf{0.233 \sd{0.010}}} & {\textbf{0.123 \sd{0.026}}}  \\
\bottomrule
\end{tabular}}
\caption{Ablation study on curvature notions. For vertex classification, accuracy ± std(\%) is reported; for vertex regression, MAE ± std is used. The best results are highlighted in \textbf{bold}.}
\label{Tab:curvature_ablation}\vspace*{-0.3cm}
\end{table}

In this experiment, we integrate existing curvature notions by replacing our learnable curvature with these predefined curvatures in our adaptive layer depth approach and evaluate their performance. We consider four curvatures from GNN literature—Ollivier-Ricci \textbf{(ORC)} \cite{southern2023expressive,toppingunderstanding}, Forman Ricci \textbf{(FRC)} \cite{fesser2024mitigating}, Jost and Liu \textbf{(JLC)} \cite{giraldo2023trade}, and Effective Resistance \textbf{(ERC)} \cite{devriendt2022discrete}—along with two structural descriptors, betweenness centrality \textbf{(BC)} and degree. These curvatures are precomputed before training and interpreted consistently: higher values indicate denser local structures, while lower values correspond to sparser regions. Note that for any edge-based curvature, vertex curvature is computed as the average of adjacent edge curvatures. The results are summarized in Table~\ref{Tab:curvature_ablation}.  

Our approach, which learns curvature in a data-driven manner, significantly outperforms all precomputed curvatures in both downstream tasks by capturing task-specific information. Additionally, we compare our method with existing curvature-based rewiring approaches and demonstrate a substantial performance improvement (please see the Appendix).

\subsubsection{Runtime and Parameter Complexity analysis}
\label{sec:runtime_analysis}

We evaluate scalability on large datasets, fixing the model depth at 4 and the hidden layer size at 256 for all methods. For our approach, we consider \( N \) in the set \(\{1,3,5\}\). We adopt GCN and GIN as the base models for ogbn-arxiv and ogbg-moltox21, respectively. Each model is assessed on parameter count, average runtime (over 100 iterations), and accuracy/RoC, as shown in Table~\ref{tbl:complexity_analysis}.

\begin{table}[h!]
\centering
\resizebox{1.0\columnwidth}{!}{
\begin{tabular}{lccc|ccc}
\toprule
\multirow{2}{*}{Methods} & \multicolumn{3}{c}{ogbn-arxiv } & \multicolumn{3}{c}{ogbg-moltox21} \\
\cmidrule(lr){2-4} \cmidrule(lr){5-7}
 & \#Param (k) & Time (s) & Acc. (\%) & \#Param (k) & Time (s) & Roc. (\%) \\
 
\midrule
Baseline   &  176  & 0.30 & 70.87 & 1118 & 4.83 & 74.84 \\
+BEC (N=1)   & 179  & 0.39 & 71.14 &  1123  & 6.90 &  75.22\\
+BEC (N=3)  & 184   & 0.52 & 71.76 & 1134 & 8.98 &  75.53\\
+BEC (N=5)  & 189 & 0.65 & 71.82 & 1144   & 10.39 &  75.76\\
\bottomrule
\end{tabular}}
\caption{Runtime and parameter complexity comparison.}
\label{tbl:complexity_analysis}\vspace*{-0.3cm}
\end{table}

As \( N \) increases, parameter count and runtime grow moderately, ensuring computational feasibility. Our approach keeps parameter growth minimal by using MLPs to estimate curvature instead of expanding the GNN’s parameter space. This design choice ensures that parameter learning of GNN remains efficient. With the same model capacity (i.e., fixed hidden layer size), our approach enhances representation power and outperforms the baseline, offering a meaningful improvement without excessive computational cost.

In the appendix, we present two additional experiments: (1) a visualization analysis, including 2D embedding visualization and curvature visualization of our approach, and (2) an evaluation of different threshold selection strategies, highlighting their influence on the depth-adaptive layer mechanism.

%% file: sections/conclusion.tex
\section{Conclusion, and Future Work}

In this work, we propose a novel mechanism for GNNs that learns vertex-wise curvature in a data-driven manner. Our framework builds on Bakry-Émery curvature, which theoretically captures geometric intricacies and information diffusion in graphs from a functional perspective.By establishing a theoretical link between feature distinction in GNNs and vertex curvature, we introduce an adaptive layer depth mechanism that leverages vertex curvature to enhance the representation capacity of GNNs. Integrated with existing GNN architectures, the proposed approach consistently improves performance across diverse downstream tasks.

A limitation of our current approach is that the threshold controlling the adaptive layer depth mechanism is set as a hyperparameter. In future work, we aim to develop a method for learning this threshold directly from the graph structure and its underlying data distribution in an end-to-end manner.

%% file: sections/appendix.tex
\appendix
\section*{Appendix}
\section{Proofs}

\thmmt*
\begin{proof}
First, we explicitly define the operators \cite{lin2010ricci}:
\[
\Gamma(f,g)(x) = \frac{1}{2}\left( L(fg) - f L g - g L f \right)(x)
\]
\[
\Gamma_2(f, f)(x) = \frac{1}{2} \left( L \Gamma(f,f)(x) - \Gamma(f, Lf)(x) \right)
\]
Under the heat semigroup evolution \( f_t = e^{-tL} f_0 \) \cite{bailleul2016heat}, for any initial function \( f_0 \in \ell^2(V) \), we have:
\[
\frac{d}{dt} |\nabla f_t|^2(x) = -2\Gamma(f_t, Lf_t)(x)
\]
Here, we differentiate the squared gradient with respect to time. The operator \( \Gamma(f_t, Lf_t)(x) \) represents the interaction between the graph Laplacian and the gradient of the function at each point in time. Next, we apply the local \( \Gamma_2 \)-criterion, which gives the lower bound:
\[
\Gamma_2(f_t, f_t)(x) \geq \kappa(x) |\nabla f_t|^2(x)
\]
This criterion tells us that the curvature \( \kappa(x) \) at each vertex provides a lower bound for the second-order derivative of the function \( f_t \) in terms of the gradient squared. Using this inequality, we obtain:
\[
\frac{d}{dt} |\nabla f_t|^2(x) \leq -2\kappa(x) |\nabla f_t|^2(x)
\]
The negative sign indicates that the gradient squared decreases over time due to the decay effect imposed by the curvature \( \kappa(x) \). This step shows how the curvature influences the rate of change of the gradient over time. Now, we apply Grönwall's inequality \cite{evans2022partial}, a standard tool for solving differential inequalities:
\[
|\nabla f_t|^2(x) \leq e^{-2\kappa(x)t} |\nabla f_0|^2(x)
\]
Grönwall's inequality gives an upper bound for the evolution of \( |\nabla f_t|^2(x) \) based on the initial condition \( |\nabla f_0|^2(x) \). The exponential decay reflects the influence of the curvature \( \kappa(x) \) on the gradient. To achieve \( |\nabla f_t|^2(x) \leq \epsilon |\nabla f_0|^2(x) \) for any \( \epsilon > 0 \), we solve for \( t \) to get:
\[
t \geq \frac{\log(1/\epsilon)}{2\kappa(x)}
\]
This bound provides the required time to ensure that the gradient squared at each vertex decays to the desired level. The time \( t \) depends on the initial condition and the local curvature at each vertex. Finally, we conclude that the local mixing time \( \tau_x(\epsilon) \), which represents the time required for the gradient to decay by a factor of \( \epsilon \), satisfies the inequality:
\[
\tau_x(\epsilon) \leq \frac{\log(1/\epsilon)}{\kappa(x)}
\]

This inequality holds for any initial condition \( f_0 \in \ell^2(V) \), where \( \ell^2(V) \) is the space of square-summable functions, meaning that \( \sum_{x \in V} |f_0(x)|^2 < \infty \). This ensures that the function \( f_0 \) has finite energy, and the solutions to \( f_t = e^{-tL} f_0 \) are well-defined for all \( t \geq 0 \).

\end{proof}

\thmfd*

\begin{proof}
From the definition of \( \tau_x(\epsilon) \) in Theorem \ref{thm:mixing}, we know that:

\[
|\nabla f_t|^2(x) \leq \epsilon |\nabla f_0|^2(x) \quad \text{for all} \quad t \geq \tau_x(\epsilon).
\]

Using the upper bound for \( \tau_x(\epsilon) \), we have:

\[
\tau_x(\epsilon) \leq \frac{\log(1/\epsilon)}{\kappa(x)}.
\]

Thus, for all \( t \geq \frac{\log(1/\epsilon)}{\kappa(x)} \), it follows that:

\[
|\nabla f_t|^2(x) \leq \epsilon |\nabla f_0|^2(x).
\]

Now, to express this in the desired form, observe that the exponential decay bound for \( |\nabla f_t|^2(x) \) is given by:

\[
|\nabla f_t|^2(x) \leq |\nabla f_0|^2(x) \exp(-\kappa(x) t).
\]

 Each layer in the GNN approximates the heat flow for a time step \( \Delta t \), so after \( l \) layers, the time is \( t = l \Delta t \). Substituting this into the definition of \( D(x,l) \), we get:
\[
D(x,l) = \frac{|\nabla f_l|^2(x)}{|\nabla f_0|^2(x)} \leq \exp(-\kappa(x) l \Delta t).
\]

For the second part, to maintain \( D(x,l) \geq \epsilon \), we require:
\[
\exp(-\kappa(x) l \Delta t) \geq \epsilon.
\]
Taking the logarithm of both sides:
\[
-\kappa(x) l \Delta t \geq \log(\epsilon),
\]
which simplifies to:
\[
l \leq \frac{\log(1/\epsilon)}{\kappa(x) \Delta t}.
\]
Thus, the number of layers must satisfy the given bound to ensure that \( D(x,l) \geq \epsilon \).
\end{proof}

\begin{figure*}[t!]
    \centering
    \subfloat[Original Graph]{
        \includegraphics[width=0.23\textwidth]{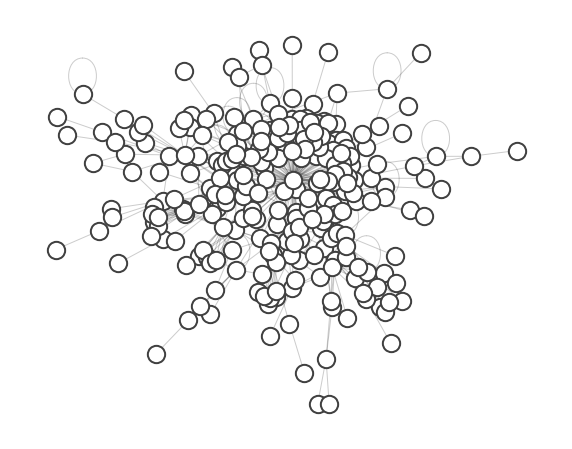}
        \label{fig:img1}
    }
    \hfill
    \subfloat[Ollivier Ricci Curvature]{
        \includegraphics[width=0.23\textwidth]{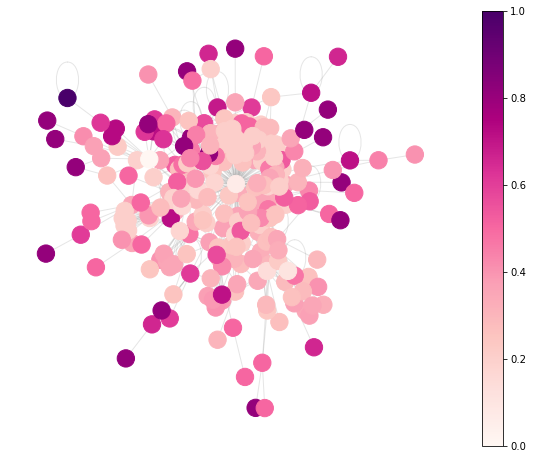}
        \label{fig:img2}
    }
    \hfill
    \subfloat[Forman Ricci Curvature]{
        \includegraphics[width=0.23\textwidth]{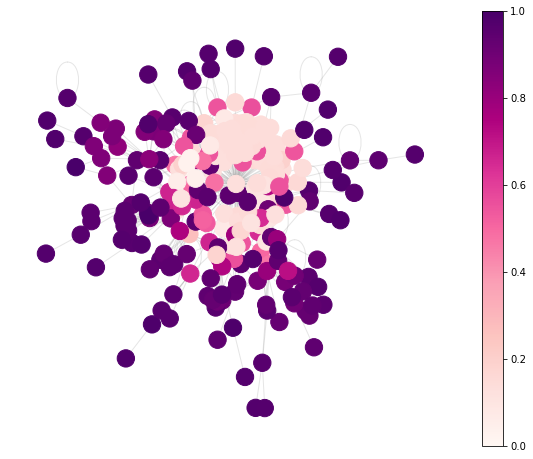}
        \label{fig:img3}
    }
    \hfill
    \subfloat[Learned Bakry-Émery Curvature]{
        \includegraphics[width=0.23\textwidth]{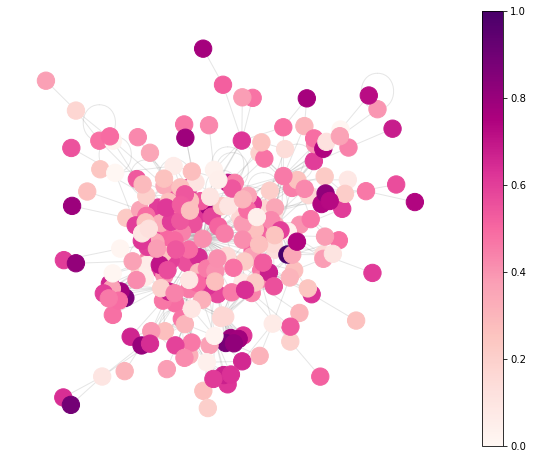}
        \label{fig:img4}
    }
    \caption{Curvature visualization for Wisconsin dataset}
    \label{fig:curvature_visualization}
\end{figure*}

\begin{figure*}[t!]
    \centering
    \subfloat[MixHop]{
        \includegraphics[width=0.23\textwidth]{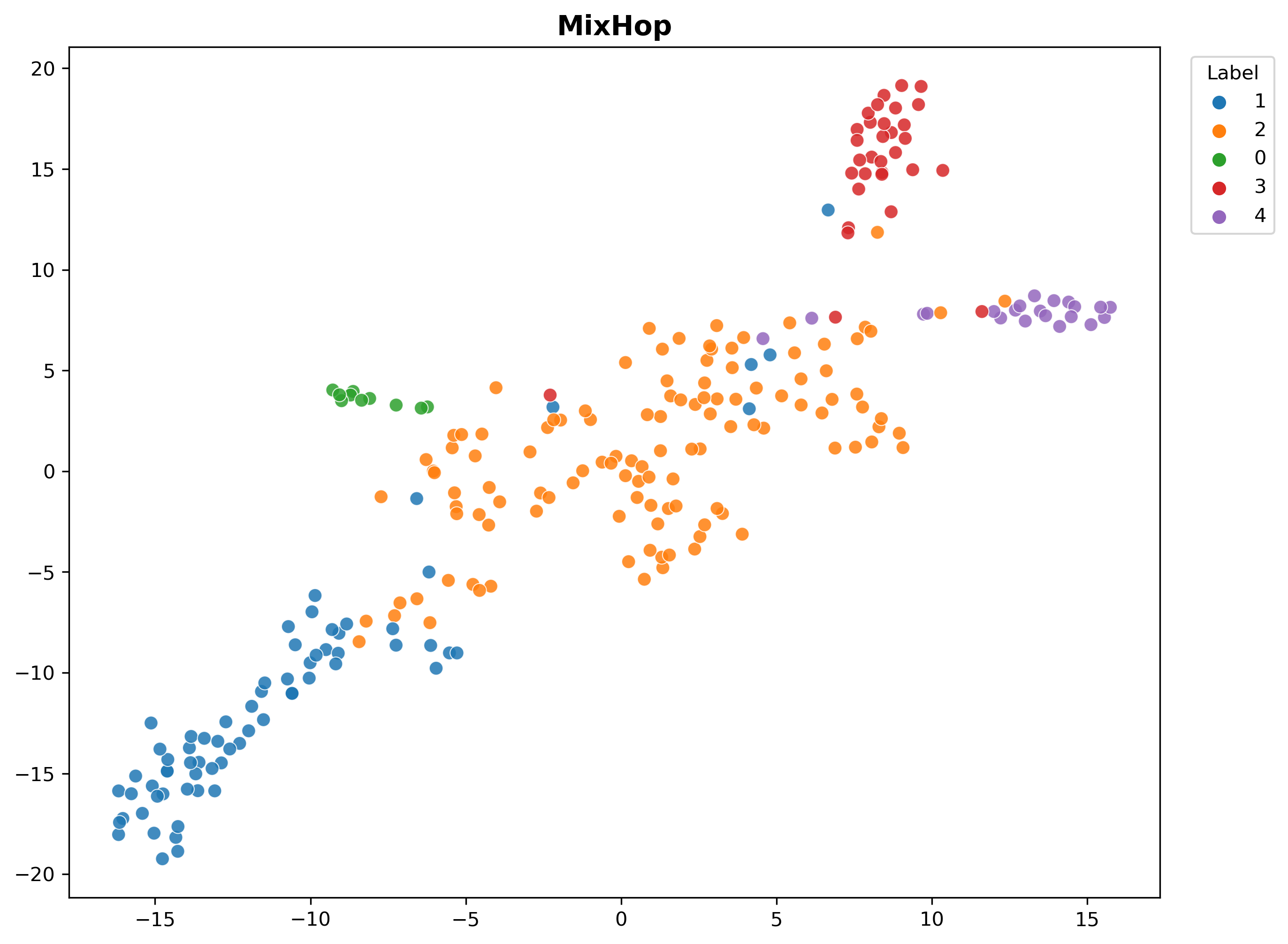}
        \label{fig:img5}
    }
    \hfill
    \subfloat[MixHop+ORC]{
        \includegraphics[width=0.23\textwidth]{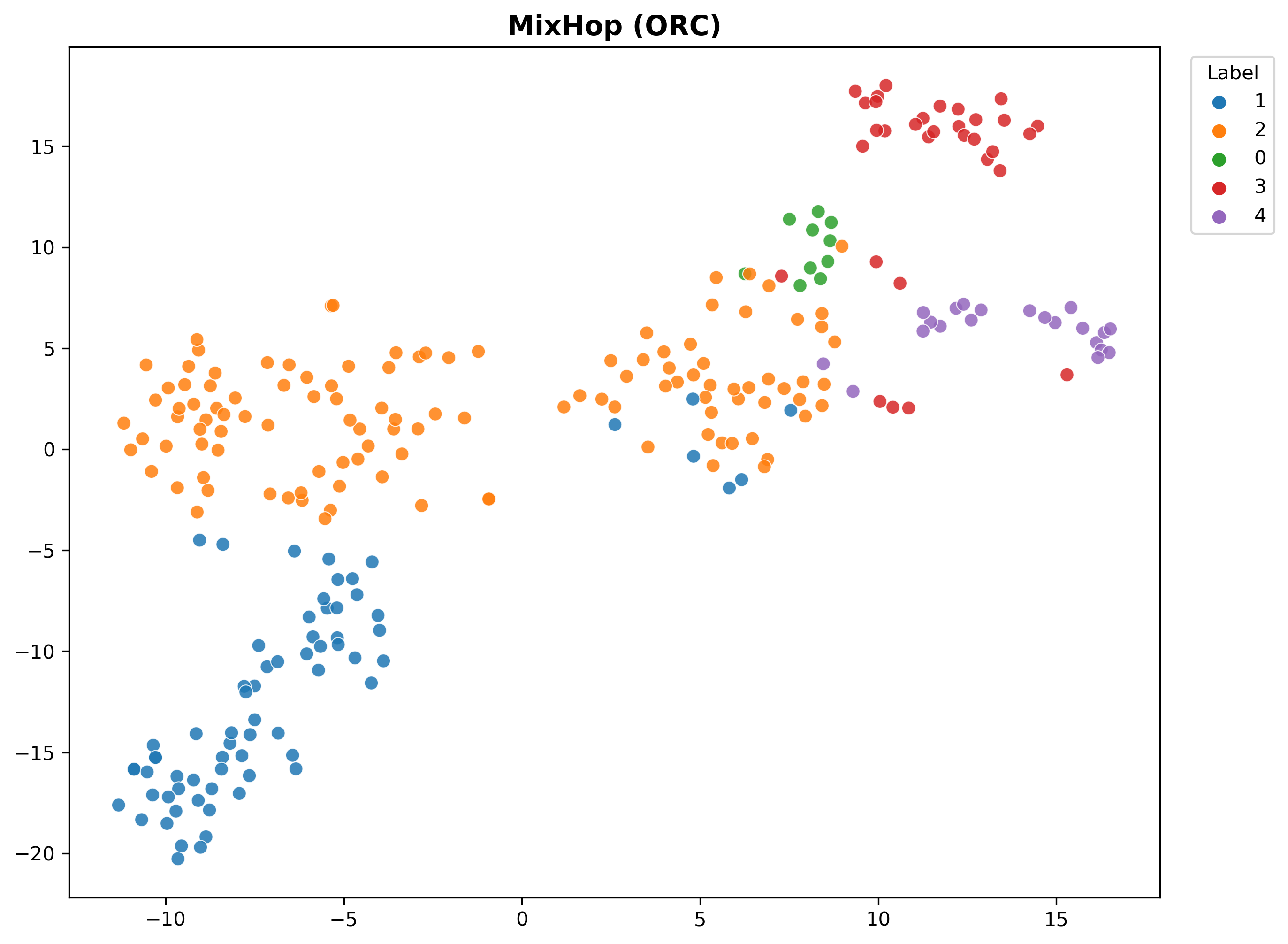}
        \label{fig:img6}
    }
    \hfill
    \subfloat[MixHop+FRC]{
        \includegraphics[width=0.23\textwidth]{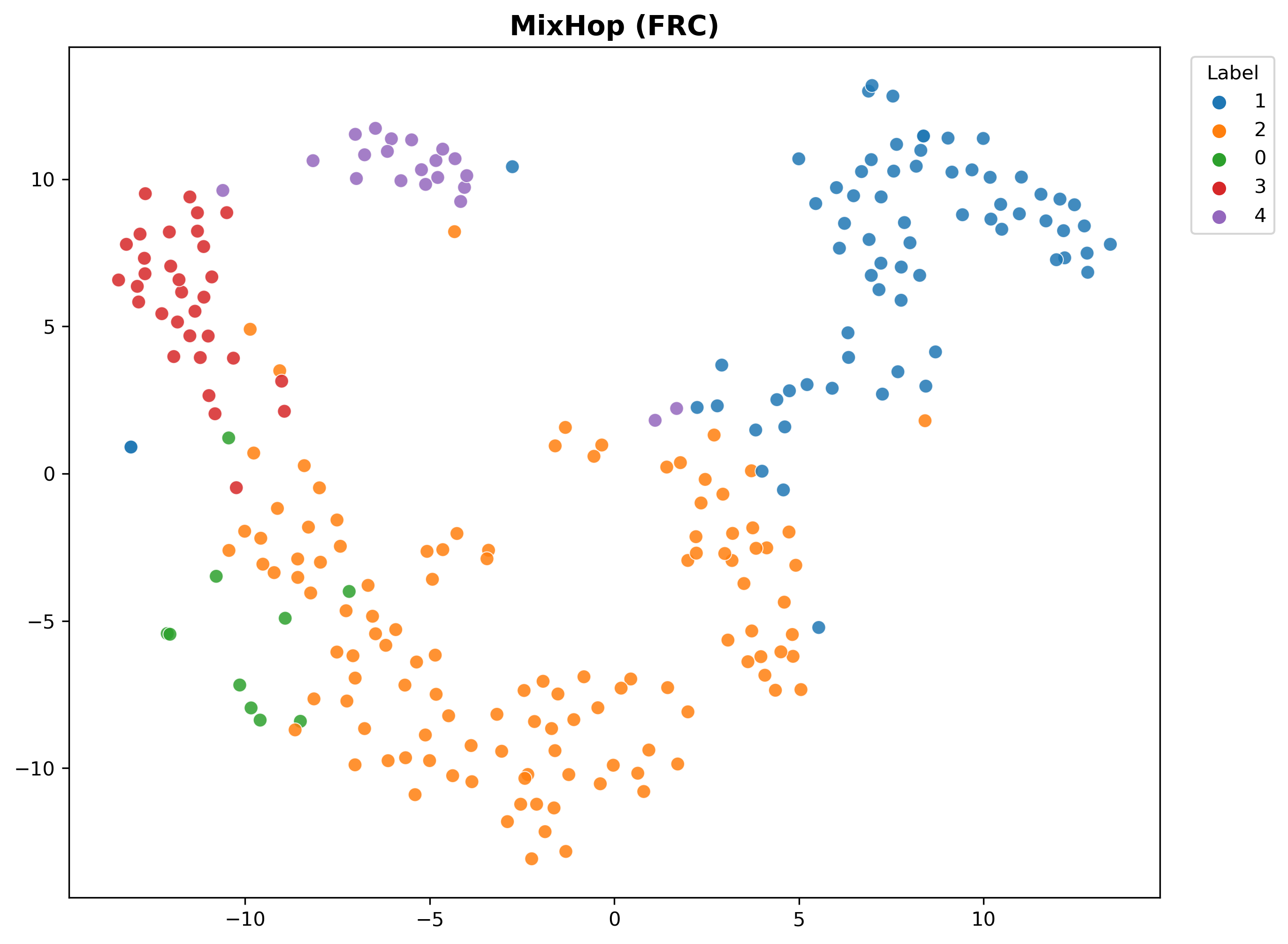}
        \label{fig:img7}
    }
    \hfill
    \subfloat[MixHop+BEC]{
        \includegraphics[width=0.23\textwidth]{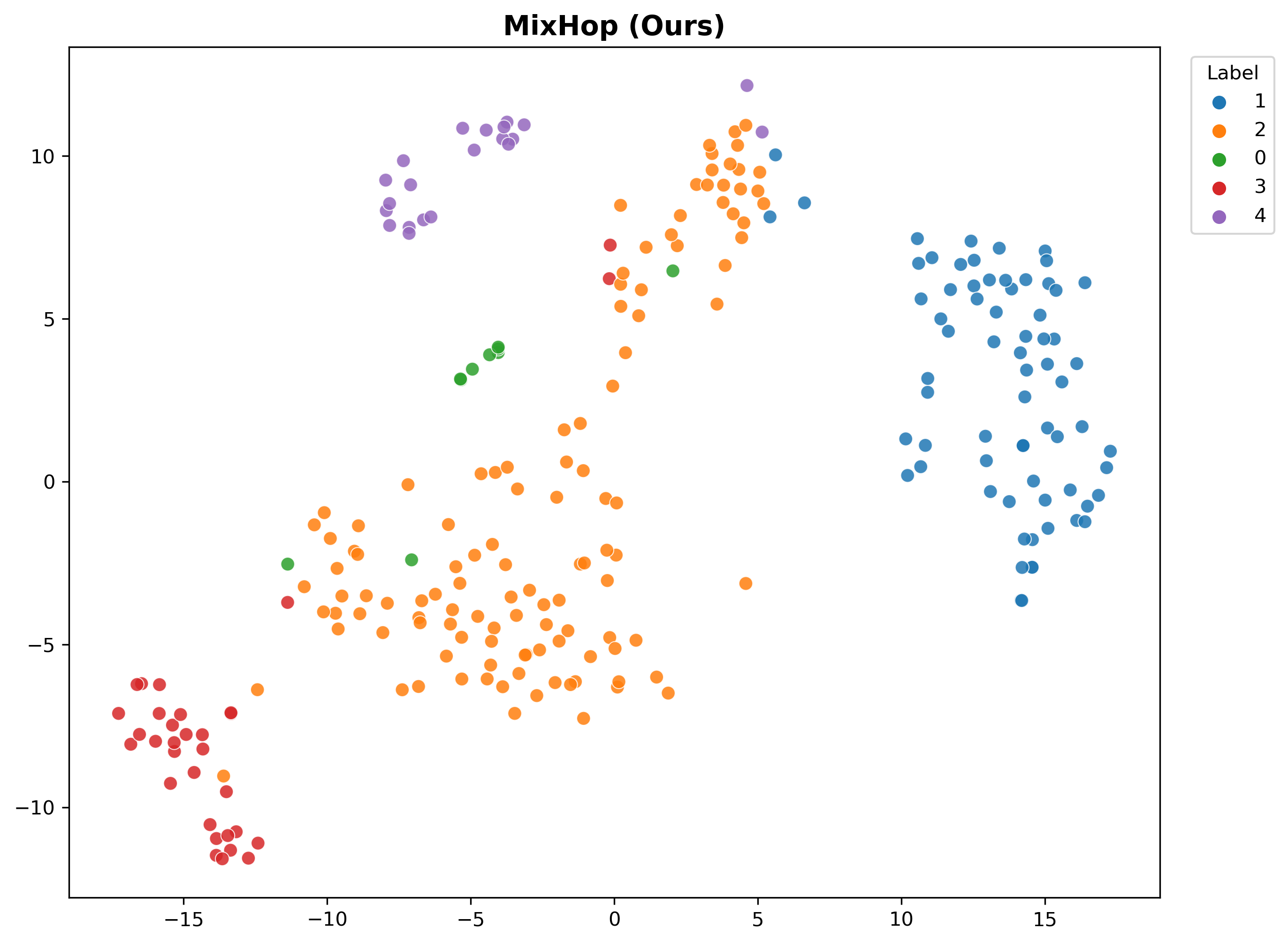}
        \label{fig:img8}
    }
    \caption{2D Vertex embedding visualization for Wisconsin dataset}
    \label{fig:embedding_visualization}
\end{figure*}

\section{Supplementary Experimental Details}

\subsection{Downstream Tasks}

We evaluate our model in the following downstream tasks.
\\
\textbf{Vertex classification:} The goal of this task is to predict the label of individual vertices in a graph, utilizing both the vertex features and the structural information of the graph.  
\\
\textbf{Vertex regression:} Vertex regression aims to predict continuous values for individual vertices in a graph. In our work, we focus on influence estimation \cite{xia2021deepis,ling2023deep}, a fundamental and challenging problem in network science. Specifically, we consider a diffusion model with an initial set of activated vertices and seek to predict the activation probability of each vertex once the diffusion process has converged. This task involves GNN leveraging the graph's structure and adapting the dynamics of the diffusion process to accurately estimate the likelihood of vertex activation.
\\
\textbf{Graph classification: }Graph classification involves predicting a single label for an entire graph based on its structure and vertex features. 
\\
\textbf{Graph regression: } Graph regression involves predicting continuous values for an entire graph, with the goal of estimating a scalar value that characterizes the graph's overall properties.

\subsection{Baselines}

We employ the following baselines for our evaluations.

\textbf{(1) Vertex classification and regression}  
\begin{itemize} \item \textbf{GCN} \cite{kipf2016semi} utilizes spectral graph convolution to aggregate neighborhood vertex features.
\item \textbf{GraphSAGE} \cite{hamilton2017inductive} incorporates neighborhood sampling and supports various aggregation functions (e.g., mean, LSTM, or sum) to enhance scalability.
\item \textbf{GAT} \cite{velivckovic2018graph} employs attention mechanisms to assign adaptive weights to neighboring vertices, enabling more effective message aggregation.
\item \textbf{SGC} \cite{wu2019simplifying} simplifies GCN by eliminating nonlinearities and collapsing weight matrices, resulting in a more efficient linear model.
\item \textbf{MixHop} \cite{abu2019mixhop} aggregates information from multiple-hop neighborhoods, capturing higher-order interactions within the graph.
\item \textbf{DIR-GNN} \cite{rossi2024edge} is a heterophilic-specific GNN that improves the message passing by incorporating directional information.
\end{itemize}

\textbf{(2) Graph classification and regression }  
\begin{itemize}
    \item \textbf{GIN} \cite{xu2018powerful} employs MLPs and sum aggregations to distinguish different graph structures, with its distinguishing power upper-bounded by the 1-WL algorithm.
    \item \textbf{GCN} \cite{kipf2016semi} is a GNN based on spectral graph convolution, and its graph distinguish power is upper-bounded by 1-WL algorithm.
    \item \textbf{ID-GNN} \cite{you2021identity} enhances graph distinguishability by incorporating vertex identity as a coloring mechanism, achieving expressiveness beyond the 1-WL algorithm.
    \item \textbf{GraphSNN} \cite{wijesinghe2022new} injects overlapping subgraph-based structural information into message passing, enhancing its expressiveness beyond the 1-WL algorithm.
    \item \textbf{NC-GNN} \cite{liuempowering} incorporates edges between neighboring vertices, in addition to neighbors, in its message-passing step, resulting in higher expressive power compared to the 1-WL algorithm.
\end{itemize}

\subsection{Model Hyper-Parameters} 

In our experiments, we search for hyper-parameters within the following ranges: the number of layers is selected from $\{2, 3, 4, 5\}$, $k$ from $\{5, 10, 15, 20, 30, 40\}$, $N$ from $\{1, 2, 3, 4, 5\}$, dropout from \\ $\{0.0, 0.1, 0.2, 0.5, 0.6, 0.9\}$, learning rate from $\{0.005, 0.009, 0.01, 0.1\}$, batch size from $\{32, 64\}$, hidden units from $\{32, 64, 128, 256, 300\}$, weight decay from $\{5e-4, 9e-3, 1e-2, 1e-1\}$, and the number of epochs from $\{100, 200, 500\}$. The Adam optimizer \cite{kingma2014:adam} is used for optimization. Further, we employ $\lambda = 1$ for all the experiments. For the ZINC and ogbg-moltox21 datasets, the learning rate decays by a factor of 0.5 after every 10 epochs.

\begin{table*}
\begin{tabular}{lcccc}
\hline
Dataset & Task Type & \# Vertices & \# Edges & \# Classes \\
\hline
Jazz & Vertex Regression & 198 & 2,742 & 1 \\
Network Science & Vertex Regression & 1,565 & 13,532 & 1 \\
Cora-ML & Vertex Regression & 2,810 & 7,981 & 1 \\
Power Grid & Vertex Regression & 4,941 & 6,594 & 1 \\
\hline
Cora & Vertex Classification & 2,708 & 5,429 & 7 \\
Citeseer & Vertex Classification & 3,327 & 4,732 & 6 \\
Pubmed & Vertex Classification & 19,717 & 44,338 & 3 \\
Actor & Vertex Classification & 7,600 & 33,544 & 5 \\
Squirrel & Vertex Classification & 5,201 & 217,073 & 5 \\
Wisconsin & Vertex Classification & 251 & 499 & 5 \\
Cornell & Vertex Classification & 183 & 295 & 5 \\
Texas & Vertex Classification & 183 & 309 & 5 \\
ogbn-arxiv & Vertex Classification & 169,343 & 1,166,243 & 40 \\
\hline
\end{tabular}
\caption{Dataset statistics for vertex classification and regression.}
\label{tbl:statistics_node_classification}
\end{table*}

\begin{table*}[ht]
\centering
\resizebox{1.5\columnwidth}{!}{
\begin{tabular}{lccccc} 
\toprule
{Datasets} & {Task Type} & \# Graphs & Avg \#Vertices & Avg \#Edges & \# Classes \\ 
\toprule
{MUTAG} & {Graph Classification} & {188} & {17.93} & {19.79} & {2} \\
{PTC\_MR} & {Graph Classification} & {344} & {14.29} & { 14.69} & {2} \\
{BZR} & {Graph Classification} & {405} & {35.75} & { 38.36} & {2} \\
{COX2} & {Graph Classification} & {467} & { 41.22} & {43.45} & {2} \\
{IMDB-BINARY} & {Graph Classification} & {1,000} & {19.77} & {96.53} & {2} \\
{PROTEINS} & {Graph Classification} & {1,113} & {39.06} & { 72.82} & {2} \\
ogbg-moltox21 & {Graph Classification} & {7,831} & {18.6} & {19.3} & {2} \\
\midrule
{ZINC} & {Graph Regression} & {12,000} & {23.1} & {49.8} & {1} \\
\bottomrule
\end{tabular}}
\caption{Dataset statistics for graph classification and regression.}
\label{tbl:statistics_graph_classification}
\end{table*}

\subsection{Benchmark Dataset Statistics}

Tables ~\ref{tbl:statistics_node_classification} 
 and ~\ref{tbl:statistics_graph_classification} present an overview of the statistics for the datasets used in our experiments. Specifically, Table~\ref{tbl:statistics_node_classification} covers statistics for vertex classification and regression, while Table ~\ref{tbl:statistics_graph_classification} provides dataset statistics related to graph classification and regression tasks.

\subsection{Computational Resources}
All the experiments were carried out on a Linux server with an Intel Xeon W-2175 2.50GHz processor (28 cores), an NVIDIA RTX A6000 GPU, and 512GB of RAM.

\subsection{Additional Experiments}

\subsubsection{Visualization Analysis}

We present a visualization analysis of vertex curvature for differnt curvature notions in Figure \ref{fig:curvature_visualization} using the Wisconsin dataset. In our approach, curvature is learned with MixHop as the base model. For two Ricci curvature methods, vertex curvature is computed as the average of adjacent edge curvatures. All the vertex curvature values are normalized between 0 and 1 for visualization purposes. The comparison reveals that our method captures heterogeneous curvature patterns with greater granularity compared to other approaches. This level of detail is well justified, as Wisconsin is a heterophilic dataset.

 Figure \ref{fig:embedding_visualization} provides vertex embedding visualization using t-SNE \cite{van2008visualizing}. MixHop is employed as the base model, and we evaluate its vertex classification performance on the Wisconsin dataset. The colors represent different class labels in the dataset. The comparison of embeddings shows significant improvements in our approach. The most notable enhancement is the formation of well-separated clusters. Our method improves cluster boundary definition, reducing overlap between class labels. Overall, our method captures both fine-grained relationships and broader community structures in embedding, leading to better classification. This is well supported by the curvature visualization, as our approach effectively captures heterogeneous curvature patterns that other methods fail to detect.

\subsubsection{Comparison with Curvature-based Rewiring Approaches}

In this experiment, we compare our approach with existing curvature-driven rewiring methods for vertex classification. We employ the experimental setup followed by \citet{fesser2024mitigating} and evaluate four rewiring techniques: AFR \cite{fesser2024mitigating}, BORF \cite{nguyen2023revisiting}, SDRF \cite{toppingunderstanding}, and FOSR \cite{karhadkarfosr}.

\begin{table}[H]
\centering
\resizebox{\columnwidth}{!}{%
\begin{tabular}{lcccccc}
\hline
Method & Cora & Citeseer & Texas & Cornell & Wisconsin \\
\hline
GCN    & 86.6 \sd{0.8} & 71.7 \sd{0.7} & 44.1 \sd{0.5} & 46.8 \sd{3.0} & 44.6 \sd{2.9} \\
GCN+AFR   & \textbf{{88.1} \sd{0.5}} & 74.4 \sd{1.0} & 51.4 \sd{0.5} & 49.7 \sd{3.4} & 52.2 \sd{2.4} \\
GCN+BORF  & 87.9 \sd{0.7} & 73.4 \sd{0.6} & 48.9 \sd{0.5} & 48.1 \sd{2.9} & 46.5 \sd{2.6} \\
GCN+SDRF & 86.4 \sd{2.1} & 72.6 \sd{2.2} & 43.6 \sd{1.2} & 43.1 \sd{1.2} & 47.1 \sd{1.0}  \\
GCN+FOSR & 86.9 \sd{2.0} & 73.5 \sd{2.0} & 46.9 \sd{1.2} & 43.9 \sd{1.1} & 48.5 \sd{1.0} \\
\midrule
GCN+BEC & 86.9 \sd{0.3} & \textbf{{77.3} \sd{0.4}} & \textbf{{64.4} \sd{2.5}}  & \textbf{{52.9} \sd{3.1}} & \textbf{{66.9} \sd{2.4}} \\
\hline
\end{tabular}%
}
\caption{Vertex classification accuracy ± std (\%) comparison with curvature-based rewiring approaches. Baseline results are taken from \citet{fesser2024mitigating}. The best results are highlighted in \textbf{bold}.}
\label{tab:rewiring_comparison}
\end{table}

As shown in Table~\ref{tab:rewiring_comparison}, our method significantly outperforms these existing approaches, highlighting that incorporating curvature directly into the learning pipeline provides a notable performance boost, compared to integrating it solely into the rewiring process.

\subsubsection{Comparison with different threshold selection mechanisms}
One of the key design decisions in our adaptive layer mechanism is determining the threshold for each layer. To gain better empirical insights into this decision, we conduct an experiment for vertex classification using GCN as the base model, with a depth of 20 layers. In this experiment, the threshold \( k \) in each layer is determined by selecting from four common distributions: \textbf{Fixed}, where \( k \) remains constant across all layers; \textbf{Power-law}, where \( k \) is sampled from a power-law distribution; \textbf{Normal (Gaussian)}, where \( k \) is sampled from a normal (Gaussian) distribution; and \textbf{Linear (increasing)}, where \( k \) increases progressively across layers. The results of this experiment are presented in Figure \ref{fig:threshold_comparison}.

\begin{figure}[H]
    \centering
    \begin{tikzpicture}[scale=0.83] 
        \begin{axis}[
            ybar,
            bar width=2.5mm, 
            xtick={0,0.1,0.2},
            xticklabels={Cora, Citeseer, Wisconsin},
            ylabel={Accuracy (\%)},
            ylabel style={font=\bfseries},
            x tick label style={rotate=0,anchor=north},
            legend style={at={(0.5,-0.25)}, anchor=north, legend columns=-1, font=\small},
            xmin=-0.05, 
            xmax=0.25,  
            enlarge x limits=0.15, 
        ]
            \addplot coordinates{(0,78.29)(0.1,70.10)(0.2,57.35)};
            \addlegendentry{Fixed}
            \addplot coordinates{(0,75.04)(0.1,69.16)(0.2,57.35)};
            \addlegendentry{Power Law}
            \addplot coordinates{(0,77.43)(0.1,70.10)(0.2,60.78)};
            \addlegendentry{Normal}
            \addplot coordinates{(0,74.25)(0.1,70.05)(0.2,48.52)};
            \addlegendentry{Linear}
        \end{axis}
    \end{tikzpicture}
    \caption{Comparison of different threshold selection mechanisms}
    \label{fig:threshold_comparison}
\end{figure}
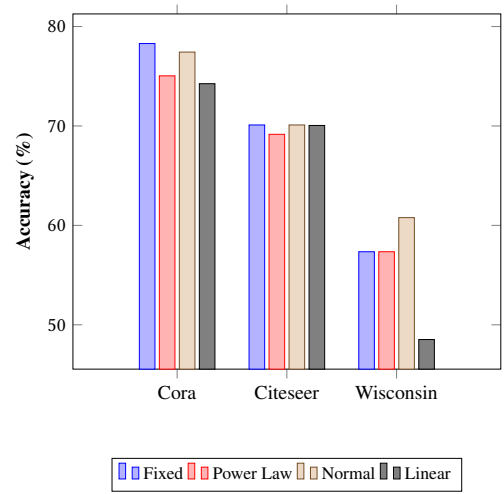

Based on the results, the fixed and normal distributions exhibit superior performance compared to the other approaches. However, it is well-established that GNNs generally perform better with shallow layers. Since it is challenging to effectively sample from a normal distribution with a limited number of samples, we opt for the Fixed threshold. This choice is not only more practical but also simple.